\documentclass{article}

\PassOptionsToPackage{numbers, sort&compress}{natbib}
\usepackage{microtype}
\usepackage{graphicx}
\usepackage{booktabs} %

\usepackage[colorlinks,citecolor=blue,urlcolor=black,linkcolor=blue,pdfborder={0 0 0}]{hyperref}
\usepackage[arxiv]{optional}

\opt{arxiv}{
	\usepackage[final]{neurips_2020}
}
\opt{neurips}{
	\usepackage{neurips_2020}

}

\usepackage{array}
\usepackage{xr}
\usepackage{bbm}
\usepackage{mathtools}
\usepackage{amsthm}
\usepackage{url}
\usepackage{comment}
\usepackage{graphicx, floatrow}
\usepackage{enumitem}
\usepackage{wrapfig}
\usepackage{framed}
\usepackage{subcaption}
\usepackage{xspace}

\usepackage[capitalize]{cleveref}  
\usepackage{datetime}

\crefname{nlem}{Lemma}{Lemmas}
\crefname{proposition}{Prop.}{Props.}
\crefname{ncor}{Corollary}{Corollaries}
\crefname{nthm}{Theorem}{Theorems}
\crefname{exa}{Example}{Examples}
\crefname{assumption}{Assumption}{Assumptions}
\crefname{appendix}{App.}{Apps.}
\crefname{equation}{}{}
\crefname{enumi}{}{}

\RequirePackage[usenames,dvipsnames]{color}
\usepackage{mathrsfs}

\usepackage{autonum}

\allowdisplaybreaks[3]

\graphicspath{{./figs/}}

\newcommand{\Lip}{\operatorname{Lip}}

\newcommand{\eps}{\epsilon}
\newcommand{\sig}{\sigma}

\def\balign#1\ealign{\begin{align}#1\end{align}}
\def\baligns#1\ealigns{\begin{align*}#1\end{align*}}
\def\balignat#1\ealign{\begin{alignat}#1\end{alignat}}
\def\balignats#1\ealigns{\begin{alignat*}#1\end{alignat*}}
\def\bitemize#1\eitemize{\begin{itemize}#1\end{itemize}}
\def\benumerate#1\eenumerate{\begin{enumerate}#1\end{enumerate}}

\newenvironment{talign*}
 {\csname align*\endcsname}
 {\endalign}
\newenvironment{talign}
 {\csname align\endcsname}
 {\endalign}

\def\balignst#1\ealignst{\begin{talign*}#1\end{talign*}}
\def\balignt#1\ealignt{\begin{talign}#1\end{talign}}
\newcommand{\qtext}[1]{\quad\text{#1}\quad} 

\let\originalleft\left
\let\originalright\right
\renewcommand{\left}{\mathopen{}\mathclose\bgroup\originalleft}
\renewcommand{\right}{\aftergroup\egroup\originalright}

\def\Matern{Mat\'ern\xspace}

\def\tinycitep*#1{{\tiny\citep*{#1}}}
\def\tinycitealt*#1{{\tiny\citealt*{#1}}}
\def\tinycite*#1{{\tiny\cite*{#1}}}
\def\smallcitep*#1{{\scriptsize\citep*{#1}}}
\def\smallcitealt*#1{{\scriptsize\citealt*{#1}}}
\def\smallcite*#1{{\scriptsize\cite*{#1}}}

\def\mbb#1{\mathbb{#1}}

\def\textsum{{\textstyle\sum}} %
\def\reals{\mathbb{R}} %

\def\naturals{\mathbb{N}} %

\def\<{\left\langle} %
\def\>{\right\rangle}

\def\defeq{\triangleq} %
\def\half{\frac{1}{2}}

\newcommand{\textfrac}[2]{{\textstyle\frac{#1}{#2}}}

\def\norm#1{\left\|{#1}\right\|} %
\newcommand{\twonorm}[1]{\norm{#1}_2} %
\newcommand{\infnorm}[1]{\norm{#1}_{\infty}} %
\newcommand{\opnorm}[1]{\norm{#1}_{\mathrm{op}}} %
\def\staticnorm#1{\|{#1}\|} %
\newcommand{\inner}[2]{\langle{#1},{#2}\rangle} %
\def\indic#1{\mbb{I}\left[{#1}\right]} %

\def\maxarg#1{\max\left({#1}\right)} %
\def\E{\mbb{E}} %
\def\Earg#1{\E\left[{#1}\right]}

\def\Esubarg#1#2{\E_{#1}\left[{#2}\right]}
\def\P{\mbb{P}} %

\newcommand{\Gsn}{\mathcal{N}}

\newcommand{\Ber}{\textnormal{Ber}}

\newcommand{\grad}{\nabla} %
\newcommand{\pderiv}[2]{\frac{\partial #1}{\partial #2}} %

\newcommand{\toas}{\stackrel{a.s.}{\to}}

\newcommand{\iid}{\textrm{i.i.d.}\xspace}
\newcommand{\dist}{\sim}
\newcommand{\distiid}{\overset{\textrm{\tiny\iid}}{\dist}}

\ifdefined\nonewproofenvironments\else
\ifdefined\ispres\else
\newtheorem{theorem}{Theorem}
\newtheorem{lemma}[theorem]{Lemma}

\renewenvironment{proof}{\noindent\textbf{Proof}\hspace*{1em}}{\qed\\}
\newenvironment{proof-sketch}{\noindent\textbf{Proof Sketch}
  \hspace*{1em}}{\qed\bigskip\\}
\newenvironment{proof-idea}{\noindent\textbf{Proof Idea}
  \hspace*{1em}}{\qed\bigskip\\}
\newenvironment{proof-of-lemma}[1][{}]{\noindent\textbf{Proof of Lemma {#1}}
  \hspace*{1em}}{\qed\\}
\newenvironment{proof-of-theorem}[1][{}]{\noindent\textbf{Proof of Theorem {#1}}
  \hspace*{1em}}{\qed\\}
\newenvironment{proof-attempt}{\noindent\textbf{Proof Attempt}
  \hspace*{1em}}{\qed\bigskip\\}

\fi

\newtheorem{proposition}[theorem]{Proposition}

\fi

\newcommand{\punt}[1]{}

\newcommand{\io}{\textrm{ i.o.}}

\newcommand{\bq}{\begin{equation}}
\newcommand{\eq}{\end{equation}}
\newcommand{\ba}{\begin{eqnarray}}
\newcommand{\ea}{\end{eqnarray}}

\newcommand{\remove}[1]{}

\newcommand{\xset}[0]{\mathcal{X}} %
\newcommand{\gset}[0]{\mathcal{G}} %
\newcommand{\steinset}[0]{\steinsetarg{}} %
\newcommand{\steinsetarg}[1]{\gset_{\norm{\cdot}_{#1}}} %
\newcommand{\gsteinset}[2]{\gset_{\norm{\cdot}_{#1},Q_n,{#2}}} %
\newcommand{\ksteinset}[1]{\gset_{#1}} %
\newcommand{\ksteinsetnorm}[2]{\gset_{#1,#2}} %
\newcommand{\knormarg}[2]{\norm{#1}_{\kset_{#2}}}
\newcommand{\knorm}[1]{\knormarg{#1}{k}}

\newcommand{\hset}[0]{\mathcal{H}} %
\newcommand{\kset}[0]{\mathcal{K}} %
\newcommand{\blset}{BL_{\norm{\cdot}}} %

\newcommand{\operator}[1]{\mathcal{T}{#1}} %
\newcommand{\oparg}[2]{(\operator{#1})({#2})} %
\newcommand{\opsub}[1]{\mathcal{T}_{#1}} %
\newcommand{\opsubarg}[3]{(\opsub{#1}{#2})({#3})} %

\newcommand{\langevin}[1]{\mathcal{T}_P{#1}} %
\newcommand{\langarg}[2]{(\langevin{#1})({#2})} %
\newcommand{\stein}[3]{\mathcal{S}({#1},{#2},{#3})} %
\newcommand{\opstein}[2]{\stein{#1}{\operator{}}{#2}} %
\newcommand{\langstein}[2]{\stein{#1}{\langevin{}}{#2}} %
\newcommand{\supnorm}[1]{\norm{#1}_{\infty}} %

\renewcommand{\norm}[1]{\staticnorm{#1}}
\renewcommand{\indic}[1]{\mbb{I}[#1]}
\newcommand{\ssd}[1]{\mathcal{SS}({#1},\operator{},\gset)} %
\newcommand{\ssdn}[1]{\mathcal{SS}({#1},\operator{},\gset_n)} %
\newcommand{\langssdn}[1]{\mathcal{SS}({#1},\langevin{},\gset_n)} %
\newcommand{\sksd}[1]{\mathcal{SS}({#1},\langevin{},{\ksteinsetnorm{k}{\norm{\cdot}}})} %
\newcommand{\twosksd}[1]{\mathcal{SS}({#1},\langevin{},{\ksteinsetnorm{k}{\twonorm{\cdot}}})} %
\newcommand{\twoksd}[1]{\mathcal{S}({#1},\langevin{},{\ksteinsetnorm{k}{\twonorm{\cdot}}})} %
\newcommand{\sset}{\sigma} %
\newcommand{\equisubs}[2]{{[#1]\choose #2}}
\newcommand{\lswass}[1]{W_{#1}} %
\newcommand{\event}{\mathcal{E}} %

\title{Stochastic Stein Discrepancies}

\author{
Jackson Gorham \\
Whisper.ai, Inc \\
\texttt{jackson@whisper.ai}
\And
Anant Raj \\
MPI for Intelligent Systems \\
T\"ubingen, Germany \\
\texttt{anant.raj@tuebingen.mpg.de}
\And 
Lester Mackey \\
Microsoft Research New England  \\
\texttt{lmackey@microsoft.com}
}

\usepackage[T1]{fontenc}
\usepackage[utf8]{inputenc}
\usepackage[english]{babel}
\usepackage{amssymb,amsmath,amsthm,amsfonts,float,bm}
\usepackage{mathtools}
\usepackage{algorithm}
\usepackage[noend]{algorithmic}

\usepackage{graphicx}
\usepackage[colorinlistoftodos]{todonotes}

\begin{document}
\maketitle

\begin{abstract}
Stein discrepancies (SDs) %
monitor convergence and non-convergence 
in approximate inference when exact integration and sampling are intractable.
However, the computation of a Stein discrepancy can be prohibitive
if the Stein operator -- often a sum over likelihood terms or potentials -- is expensive to evaluate. 
To address this deficiency, we show that \emph{stochastic Stein discrepancies} (SSDs) based on subsampled approximations of the Stein operator inherit the convergence control properties of standard SDs with probability $1$.
Along the way, we establish the convergence of Stein variational gradient descent (SVGD) on unbounded domains, resolving an open question of Liu (2017).
In our experiments with biased Markov chain Monte Carlo (MCMC) hyperparameter tuning, approximate MCMC sampler selection, and stochastic SVGD, 
SSDs deliver comparable inferences to standard SDs with orders of magnitude fewer likelihood evaluations.
\end{abstract}
\section{Introduction}\label{sec:intro}
Markov chain Monte Carlo (MCMC) methods \cite{BrooksGeJoMe11} provide asymptotically correct sample estimates $\textfrac{1}{n}\sum_{i=1}^n h(x_i)$ of the complex integrals $\Esubarg{P}{h(Z)} = \int h(z)dP(z)$ that arise in Bayesian inference, maximum likelihood estimation \cite{Geyer91}, and probabilistic inference more broadly.
However, MCMC methods often require cycling through a
large dataset or a large set of factors to produce each new sample point $x_i$.
To avoid this computational burden, many have turned to scalable approximate MCMC methods
\citep[e.g.][]{WellingTe11, Ahn2012, PattersonTe13, Chen2014, Dubois2014},
which mimic standard MCMC procedures while using only a small subsample of datapoints to generate each new sample point. 
These techniques reduce Monte Carlo variance by delivering larger sample sizes in less time but sacrifice %
asymptotic correctness %
by introducing a persistent bias. 
This bias 
creates new difficulties for sampler monitoring, selection, and hyperparameter tuning, as standard MCMC diagnostics, like trace plots and effective
sample size, 
rely upon asymptotic exactness.
To effectively assess the quality of approximate MCMC outputs,
a line of work \citep{GorhamMa15,MackeyGo16,GorhamDuVoMa19,GorhamMa17,HugginsMa2018,ChenMaGoFXOa18} developed computable \emph{Stein discrepancies} (SDs) that quantify the maximum discrepancy between sample and target expectations and provably track sample convergence to the target $P$, even when explicit integration and direct sampling from $P$ are intractable.
SDs have since been used to compare approximate MCMC procedures \citep{AicherPuNeFeFo19}, test goodness of fit \cite{ChwialkowskiStGr2016,LiuLeJo16,jitkrittum2017linear,HugginsMa2018}, train generative models \cite{WangLi2016,PuGaHeChHaCa2017}, 
generate particle approximations
\cite{ChenMaGoFXOa18,futami2019bayesian,ChenBaFXGoGiMaOa19},
improve particle approximations \cite{LiuWa2016stein,LiuLe2016,hodgkinson2020reproducing},  compress samples \cite{RiabizChCoSwNiMaOa2020}, conduct variational inference \cite{Ranganath2016}, and estimate parameters in intractable models \cite{BarpFXDuGiMa2019}.

However, the computation of the Stein discrepancy itself can be prohibitive if the Stein operator applied at each datapoint -- often a sum over datapoint likelihoods or factors -- is expensive to evaluate.
This expense has led some users to 
heuristically approximate Stein discrepancies by subsampling data points \cite{LiuWa2016stein,Ranganath2016,AicherPuNeFeFo19}. In this paper, we formally justify this practice by proving that \emph{stochastic Stein discrepancies} (SSDs) based on subsampling inherit the desirable convergence-tracking properties of standard SDs with probability $1$.
We then apply our techniques to analyze a scalable stochastic variant of the popular Stein variational gradient descent (SVGD) algorithm \citep{LiuWa2016stein} for particle-based variational inference.
Specifically, we generalize the compact-domain convergence results of \citet{liu2017stein} to show, first, that SVGD converges on unbounded domains and, second, that stochastic SVGD (SSVGD) converges to the same limit as SVGD with probability $1$.
We complement these results with a series of experiments illustrating the application of SSDs 
to biased MCMC hyperparameter tuning, approximate MCMC sampler selection, and particle-based variational inference. 
In each case, we find that SSDs deliver inferences equivalent to or more accurate than standard SDs with orders of magnitude fewer datapoint accesses.

The remainder of the paper is organized as follows. 
In \cref{sec:sample_quality}, we review standard desiderata and past approaches for measuring the quality of a sample approximation.
In \cref{sec:st_stein}, we provide a formal definition of stochastic Stein discrepancies for scalable sample quality measurement and present a stochastic SVGD algorithm for scalable particle-based variational inference.
We provide probability $1$ convergence guarantees for SSDs and SSVGD in \cref{sec:theory} and demonstrate their practical value in
\cref{sec:experiments}. We discuss our findings and posit directions for future work in \cref{sec:conclusion}.

\paragraph{Notation}
For vector-valued $g$ on $\xset\subseteq \reals^d$, we
define the expectation $\mu(g) \defeq \int g(x) d\mu(x)$ for each probability measure $\mu$, 
the divergence $\inner{\grad}{g(x)} \defeq \sum_{j=1}^d \pderiv{}{x_j}g_j(x)$, and 
the $\twonorm{\cdot}$ boundedness and Lipschitzness parameters  $\supnorm{g} \defeq
\sup_{x\in\reals^d}\twonorm{g(x)}$ and $\Lip(g) \defeq
\sup_{x\neq y\in\xset}\frac{\twonorm{g(x) - g(y)}}{\twonorm{x-y}}$. 
For any matrix $A$, let $\opnorm{A}\defeq \sup_{\twonorm{x}\le 1} \twonorm{A x}$ be
the operator norm of $A$.
For any $L\in\naturals$, we write $[L]$ for $\{1,\dots,L\}$.
We write $\Rightarrow$ for the weak convergence and $\toas$ for almost sure
convergence of probability measures. We denote the set of continuous
functions and continuously differentiable functions on
$\xset$ as $C(\xset)$ and $C^1(\xset)$ respectively, and use the shorthand
$C$ and $C^1$ whenever $\xset=\reals^d$. We also denote the set of functions on
$\reals^d\times\reals^d$ continuously differentiable in both arguments by
$C^{(1,1)}$.
\section{Measuring Sample Quality} \label{sec:sample_quality}
Consider a target distribution $P$ supported on
$\xset\subseteq\reals^d$. We assume that
exact expectations under $P$ are unavailable for many functions of interest,
so we will an employ a discrete measure $Q_n \defeq
\textfrac{1}{n}\sum_{i=1}^n \delta_{x_i}$ based on a sample $(x_i)_{i=1}^n$
to approximate expectations under $P$.  Importantly, we will make no
assumptions on the origins or nature of the sample points $x_i$; they may be
the output of \iid sampling, drawn from an arbitrary Markov chain, or even
generated by a deterministic quadrature rule.

To assess the usefulness of a given sample, we seek a quality measure that quantifies how well expectations under $Q_n$ match those under $P$.
At the very least, this quality measure should
\emph{(i)} determine when $Q_n$ converges to the
target $P$, \emph{(ii)} determine when $Q_n$ does not converge to $P$, and \emph{(iii)} be computationally tractable.
\emph{Integral probability
  metrics} (IPMs) \citep{Muller97} are natural candidates, as they measure the maximum absolute difference in
expectation between probability measures $\mu$ and $\nu$ over a set of test functions $\hset$:
\begin{equation}
  d_{\hset}(\mu, \nu) \defeq \sup_{h\in\hset} |\Esubarg{\mu}{h(X)} - \Esubarg{\nu}{h(Z)}|.
\end{equation}
Moreover, for many IPMs, like the Wasserstein distance ($\hset = \{h:\xset\to\reals\, |\, \Lip(h) \le 1\}$) and the Dudley metric ($\hset = \{h:\xset\to\reals\, |\, \supnorm{h} + \Lip(h) \le 1\}$), convergence of $d_{\hset}(Q_n, P) \to 0$ implies that $Q_n\Rightarrow P$, in satisfaction of Desideratum (ii).
Unfortunately, these same IPMs typically cannot be computed without exact integration under $P$. 
\citet{GorhamMa15} circumvented this issue by constructing a new family of IPMs -- \emph{Stein discrepancies} -- from test functions known a priori to be mean zero under $P$.
Their construction was inspired by Charles Stein's three-step method for proving central limit theorems \cite{Stein72}:
\begin{enumerate}[leftmargin=.7cm]
\item Identify an operator $\operator{}$ that generates mean-zero functions on its domain $\gset$:
  \begin{align}\label{eq:mean-zero-operator}
    \Esubarg{P}{\oparg{g}{Z}} = 0 \text{ for any } g\in\gset.
  \end{align}
  The chosen \emph{Stein operator} $\operator{}$ and \emph{Stein set} $\gset$ together yield an IPM-type measure which eschews explicit integration under $P$:
\begin{align}\label{eqn:sd}
    \opstein{\mu}{\gset} \defeq d_{\operator{\gset}}(\mu, P)
      = \sup_{g\in\gset}|\Esubarg{\mu}{(\operator{g})(X)} - \Esubarg{P}{(\operator{g})(Z)}|
      = \sup_{g\in\gset}|\Esubarg{\mu}{(\operator{g})(X)}|.
  \end{align}
\citet{GorhamMa15} named this measure the \emph{Stein discrepancy}. %
  
\item Lower bound the Stein discrepancy by an IPM known to dominate
  convergence in distribution. This is typically done for a large class
  of targets once and thus ensures that $\opstein{Q_n}{\gset} \to
  0$ implies $Q_n \Rightarrow P$ (Desideratum \emph{(ii)}).
\item Upper bound the Stein discrepancy to ensure that the Stein discrepancy
  $\opstein{Q_n}{\gset}\to 0$ when $Q_n$ converges suitably to
  $P$ (Desideratum \emph{(i)}).
\end{enumerate}
Prior work has instantiated a variety of Stein operators $\operator{}$ and Stein sets $\gset$ satisfying Desiderata \emph{(i)-(iii)} for large classes of target distributions
\cite{Stein72, SteinDiHoRe04, GorhamMa15, MackeyGo16, GorhamDuVoMa19, GorhamMa17, HugginsMa2018, ChenMaGoFXOa18,ChenBaFXGoGiMaOa19,erdogdu2018global}. 
We will focus on \emph{decomposable operators}: $\operator{} = \sum_{l=1}^L \opsub{l}$ that 
decompose as a sum of $L$ base operators $\opsub{l}$ that are less expensive to evaluate than $\operator{}$. 
A prime example is the \emph{Langevin Stein operator} derived in \cite{GorhamMa15},
\begin{equation}\label{eqn:def-langevin-stein-operator}
  \langarg{g}{x} = \inner{\grad\log p(x)}{g(x)} + \inner{\grad}{g(x)},
\end{equation}
applied to a differentiable posterior density $p(x) \propto \pi_0(x)\prod_{l=1}^L \pi(y_l | x)$ on $\reals^d$ for $\pi_0$ a prior density, $\pi(\cdot | x)$  a  likelihood function, and $(y_l)_{l=1}^L$ a sequence of observed datapoints.
In this case, the Langevin operator $\langevin{} = \sum_{l=1}^L \opsub{l}$ for 
\balignt\label{eq:langevin-posterior}
(\opsub{l}g)(x) =  \inner{\grad\log
  p_l(x)}{g(x)} + \textfrac{1}{L}\inner{\grad}{g(x)}
\qtext{and} 
p_l(x) \defeq \pi_0(x)^{1/L}\pi(y_l | x),
\ealignt
so that each base operator involves accessing only a single datapoint.

\section{Stochastic Stein Discrepancies} \label{sec:st_stein}

Whenever the Stein operator decomposes as $\operator{} = \sum_{l=1}^L \opsub{l}$, the standard Stein discrepancy (SD) objective \cref{eqn:sd} demands that every base operator $\opsub{l}$ be evaluated at every sample point $x_i$; this cost can quickly become prohibitive if $L$ and $n$ are large. %
To alleviate this burden, we will consider a new class of discrepancy measures based on subsampling base operators. 
We emphasize that our aim in doing so is \emph{not} to approximate standard SDs but rather to develop more practical alternative discrepancy measures that control convergence in their own right.
To this end, we fix a batch size $m$ and, for each $i \in [n]$, independently select a uniformly random subset $\sset_i$ of size $m$ from $[L]$. 
Then for any
$\gset$, we define the \emph{stochastic Stein discrepancy} (SSD) as the random quantity
\begin{talign}\label{ssd_subset}
    \ssd{Q_n} 
        &\defeq \sup_{g\in\gset}\left |\frac{1}{n}\sum_{i=1}^n
    \frac{L}{m}\opsubarg{\sset_i}{g}{x_i}\right |,
\end{talign}
where, for each $\sset\subseteq [L]$, we introduce the subset operator
$\opsub{\sset} \defeq \sum_{l \in \sset} \opsub{l}$.
In our running example of the Langevin posterior decomposition \cref{eq:langevin-posterior}, we have
\balignt
(\opsub{\sigma}g)(x) =  \inner{\grad\log
  p_\sigma(x)}{g(x)} + \textfrac{m}{L}\inner{\grad}{g(x)}
\qtext{for} 
p_\sigma(x) \defeq \pi_0(x)^{m/L} \prod_{l\in\sigma} \pi(y_l | x),
\ealignt
so that each subset operator processes only a minibatch of $m$ datapoints.

By construction, the SSD reduces the number of base operator evaluations by
a factor of $m/L$.  Nevertheless, we will see in the \cref{sec:theory} that
SSDs inherit the convergence-determining properties of standard SDs with
probability $1$.
Notably, the continued detection of convergence and non-convergence to $P$ is made possible by the use of an independent subset $\sigma_i$ per sample point.
If, for example, the same minibatch of $m$ datapoints were used for all sample points instead, then the resulting discrepancy would determine convergence to an incorrect posterior conditioned on that minibatch rather than to the desired target $P$.

\subsection{Stochastic kernel Stein discrepancies}

Before turning to the convergence theory we pause to highlight a second property of practical import: when the Stein set is a unit ball of a reproducing kernel Hilbert space (RKHS), the SSD \cref{eqn:ssd-definition} admits a closed-form solution.  
We illustrate this for the Langevin operator \cref{eqn:def-langevin-stein-operator}
and the \emph{kernel Stein set}~\cite{GorhamMa17}
\balignt\label{eq:kernel-stein-set}
\ksteinsetnorm{k}{\norm{\cdot}} \defeq \{ g = (g_1,\dots, g_d) \mid \norm{v}^*\le 1 \text{ for } v_j \defeq \knorm{g_j} \}
\ealignt
with arbitrary vector norm $\norm{\cdot}$ and $\knorm{\cdot}$ the RKHS norm of a reproducing kernel $k$ on $\reals^d\times\reals^d$.

\begin{proposition}[SKSD closed form]\label{sksd-formula}
If $k\in C^{(1,1)}$, then $\sksd{Q_n} =\norm{w}$ where, $\forall j \in [d]$,
\balignt
w_j^2 \defeq \frac{1}{n^2} \sum_{i=1}^n\sum_{i'=1}^n (\frac{L}{m}\grad_{x_{ij}}\log p_{\sig_i}(x_i) + \grad_{x_{ij}})(\frac{L}{m}\grad_{x_{i'j}}\log p_{\sig_{i'}}(x_{i'}) +\grad_{x_{i'j}})k(x_i,x_{i'}).
\ealignt
\end{proposition}

We call such discrepancies \emph{stochastic kernel Stein discrepancies} (SKSDs) in homage to the standard kernel Stein discrepancies (KSDs) introduced in \citep{ChwialkowskiStGr2016,LiuLeJo16,GorhamMa17}.
See \cref{sec:proof-sksd-formula} for the proof of \cref{sksd-formula}.

\paragraph{Related work} 
Several research groups have stochastically approximated kernel-based Stein
sets $\gset$ to reduce the $\Omega(n^2)$ computational expense of goodness-of-fit testing
\citep{jitkrittum2017linear,HugginsMa2018}, measuring sample quality
\citep{HugginsMa2018}, and improving sample quality with Stein variational
gradient descent \citep{li2019stochastic} while leaving the original
operator $\operator{}$ unchanged.
Others have improved communication efficiency by deploying standard SDs with special Stein sets featuring low-dimensional coordinate-dependent kernels \citep{WangZeLi18,ZhuoLiShZhChZh18}.
Here we focus on the distinct and complementary burden of evaluating an expensive Stein operator $\operator{}$ at each sample point and note that the aforementioned approaches can be combined with datapoint subsampling to obtain substantial speed-ups.
The recent, independent work of \citet{hodgkinson2020reproducing} uses a Langevin SKSD (in our terminology) to learn approximate importance sampling weights for an initial sample $Q_n$.  Thm.~1 of \citep{hodgkinson2020reproducing} shows that the reweighted version of $Q_n$ asymptotically minimizes the associated KSD provided that the sample points $x_i$ are drawn from a $V$-uniformly ergodic Markov chain. In contrast, we offer convergence guarantees in \cref{sec:theory} that apply to arbitrary sample points $x_i$, allow for non-kernel Stein discrepancies, and accommodate more general decomposable operators.

\subsection{Stochastic Stein variational gradient descent}
Our SSD analysis will also yield convergence guarantees for a stochastic version of the popular Stein variational gradient descent (SVGD) algorithm \citep{LiuWa2016stein} on $\reals^d$.
SVGD iteratively improves a particle approximation $Q_n$ to a target distribution $P$ by moving each particle in the direction
\begin{talign}
    g_{Q_n}^*(z) = \frac{1}{n}\sum_{j=1}^n k(x_j, z) \grad\log p(x_j)  + \grad_{x_j} k(x_j, z)
\end{talign}
 that minimizes the KSD $\twoksd{Q_n}$ with Langevin operator \cref{eqn:def-langevin-stein-operator}.  
 However, when $\grad \log p = \sum_{l=1}^L \grad \log p_l$ is the sum of a large number of independently evaluated terms, each SVGD update can be prohibitively expensive.
 A natural alternative is to move each particle in the direction that minimizes the stochastic KSD $\twosksd{Q_n}$,
 \begin{talign}
    g_{Q_n,m}^*(z) = \frac{1}{n}\sum_{j=1}^n\frac{L}{m} k(x_j, z) \grad\log p_{\sig_j}(x_j) + \grad_{x_j} k(x_j, z).
\end{talign}
This amounts to replacing each $\grad \log p(x_j)$ evaluation with an independent minibatch estimate on each update round to reduce the per-round gradient evaluation cost from $Ln$ to $mn$.
 The resulting \emph{stochastic Stein variational gradient descent (SSVGD)}   algorithm is detailed in \cref{alg:ssvgd}.
 Notably, after introducing SVGD, \citet{LiuWa2016stein} recommended subsampling as a heuristic approximation to speed up the algorithm. In \cref{sec:ssvgd-converges}, we aim to formally justify this practice.
\begin{algorithm}[tb]
\caption{Stochastic Stein Variational Gradient Descent (SSVGD)}
\label{alg:ssvgd}
\begin{algorithmic}
\STATE {\bf Input:} Particles $(x_i^0)_{i=1}^n$, target $\grad \log
p = \sum_{l\in [L]} \grad \log p_l$, kernel $k$, batch size $m$, rounds $R$
\FOR{$r = 0, \cdots, R-1$}
\STATE For each $j\in[n]$: sample independent mini-batch $\sig_{j}$ of size $m$ from $[L]$
\STATE For each $i\in[n]$: $x_i^{r+1} \gets x_i^r + \eps_r \frac{1}{n}\sum_{j=1}^n \frac{L}{m} k(x_j^r, x_i^r)\grad \log p_{\sig_j}(x_j^r) + \grad_{x_j^r} k(x_j^r, x_i^r)$
\ENDFOR
\STATE {\bf Output:} Particle approximation $Q_{n,R}^m = \frac{1}{n} \sum_{i=1}^n \delta_{x_i^R}$  of the target $P$ 
\vspace{-.39\baselineskip}
\end{algorithmic}
\end{algorithm}
\section{Convergence Guarantees}\label{sec:theory}
In this section, we begin by showing that appropriately chosen SSDs detect the convergence and non-convergence of $Q_n$ to $P$ with probability $1$ and end with new convergence results for SVGD and SSVGD.
The former results will allow for an evolving sequence of Stein sets $(\gset_n)_{n=1}^{\infty}$ to accommodate the graph Stein sets of \cite{GorhamMa15,GorhamDuVoMa19}.
While we develop the most extensive theory for the popular Langevin Stein
operator \cref{eqn:def-langevin-stein-operator} with domain $\xset=\reals^d$, our results on detecting
convergence (\cref{ssds-detect-convergence}), detecting bounded
non-convergence (\cref{ssds-detect-bounded-sd-non-convergence}), and
enforcing tightness (\cref{ssd-tightness}) apply to any decomposable Stein
operator on any convex subset $\xset\subseteq\reals^d$.
Throughout, we use the shorthand
$\equisubs{L}{m} \defeq
\{\sset\subseteq [L] : |\sset| = m\}$ to indicate all subsets of $[L]$ of size $m$.
\subsection{Detecting convergence with SSDs}
We say that an SSD {detects convergence} if $\ssdn{Q_n} \to 0$ whenever $Q_n$ converges to $P$ in a standard probability metric, like the Wasserstein distance $\lswass{a}(Q_n, P)\defeq \inf_{X\sim Q_n, Z\sim P} \E[\twonorm{X-Z}^a]^{1/a}$ for $a\geq 1$.
Our first result, proved in \cref{sec:proof-ssds-detect-convergence}, shows
that an SSD detects Wasserstein convergence with probability $1$ if its base
operators $\opsub{\sset}$ generate continuous functions that grow no more
quickly than a polynomial and have locally bounded
derivatives. \cref{ssds-detect-convergence} is broad enough to cover all of
the Stein operator-set pairings with SD convergence-detection results in
\citep{GorhamMa15,GorhamDuVoMa19,GorhamMa17}.

\begin{theorem}[SSDs detect convergence]
\label{ssds-detect-convergence}
Suppose that for some $a, c > 0$ and each $\sset\in\equisubs{L}{m}$ and $n\geq 1$,
 $\opsub{\sset}{\gset_n}\subset C(\xset)$, $\sup_{g \in \gset_n}
|\opsubarg{\sset}{g}{x}| \leq c(1+\twonorm{x}^a)$, $\sup_{n\ge 1, g \in \gset_n,
  x,y\in K}
\textfrac{|\opsubarg{\sset}{g}{x} - \opsubarg{\sset}{g}{y}|}{\twonorm{x-y}} < \infty$
for each compact set $K$, and $P(\operator{g}) = 0$ for all $g\in\gset_n$.
If $\lswass{a}(Q_n, P)\defeq \inf_{X\sim Q_n, Z\sim P} \E[\twonorm{X-Z}^a]^{1/a} \to 0$, then $\ssdn{Q_n} \toas 0$. 
\end{theorem}

\subsection{Detecting non-convergence with SSDs}
We say that an SSD {detects non-convergence} if $Q_n \not\Rightarrow P$ implies $\ssdn{Q_n} \not\to 0$.
To establish this property, we first associate with every SSD, $\ssdn{Q_n}$, a \emph{bounded Stein discrepancy}, %
\balignt \label{eq:bounded-stein-set}
\opstein{Q_n}{\gset_{b,n}}\to 0
\qtext{with}
\gset_{b,n} \defeq \{ g \in \gset_n: \supnorm{\opsub{\sset}g}\leq 1, \forall \sset\in\equisubs{L}{m}\},
\ealignt 
in which each Stein function is 
constrained to be bounded under each subset operator $\opsub{\sset}$.
We then show that SSDs detect non-convergence (culminating in \cref{coercive-ssds-detect-non-convergence})
in a series of steps:
\begin{enumerate}
    \item \cref{bounded-sds-detect-tight-non-convergence}: If $Q_n \not\Rightarrow P$ then either
a bounded SD $\not\to 0$ or $(Q_n)_{n=1}^\infty$ is not tight.
    \item \cref{ssds-detect-bounded-sd-non-convergence}: If a bounded
SD $\not\to 0$ then, with probability $1$, its SSD $\not\to 0$.
    \item \cref{ssd-tightness}: If $(Q_n)_{n=1}^\infty$ is
not tight, then the SSD $\not\to 0$ surely.
\end{enumerate}

We begin by showing that, for the popular Langevin operator
\cref{eqn:def-langevin-stein-operator} and each Stein set analyzed in
\citep{GorhamMa15,GorhamDuVoMa19,GorhamMa17,ChenMaGoFXOa18,ChenBaFXGoGiMaOa19}, bounded SDs
detect \emph{tight} non-convergence.  That is, if $Q_n \not\Rightarrow P$,
then either $\opstein{Q_n}{\gset_{b,n}}\not\to 0$ or some mass in the
sequence $(Q_n)_{n=1}^\infty$ escapes to infinity. The proof is in
\cref{sec:proof-bounded-sds-detect-tight-non-convergence}.
\begin{theorem}[Bounded SDs detect tight non-convergence]
\label{bounded-sds-detect-tight-non-convergence}
Consider the Langevin Stein operator $\langevin{}$ \cref{eqn:def-langevin-stein-operator} 
with Lipschitz $\grad \log p$ satisfying  \emph{distant dissipativity}~\citep{Eberle2015,GorhamDuVoMa19} for some $\kappa > 0$ and $r \geq 0$:
\balignt\label{eq:drift-growth}
\inner{\grad \log p(x)-\grad \log p(y)}{x-y} \leq -\kappa \twonorm{x-y}^2 + r, \qtext{for all} x,y\in \xset = \reals^d.
\ealignt
Suppose $\sup_{x\in\xset}{\twonorm{\grad\log p_{\sset}(x)}}{/(1 + \twonorm{x})} < \infty$ for each $\sset\in\equisubs{L}{m}$,
fix a sequence of probability measures $(Q_n)_{n=1}^\infty$, 
and 
consider the bounded Stein set 
$\gset_{b,n}$ \cref{eq:bounded-stein-set}
for any of the following sets $\gset_n$:
\begin{enumerate}[leftmargin=.9cm,label=(\textbf{A.\arabic*})]
    \item \label{kernel-set} $\gset_n = \ksteinsetnorm{k}{\norm{\cdot}}$ \cref{eq:kernel-stein-set}, the \emph{kernel Stein set} of~\cite{GorhamMa17} with $k(x,y) = \Phi(x-y)$ for $\Phi \in C^2$ with non-vanishing Fourier transform.
    \item\label{classical-set} $\gset_n = \steinset 
		\defeq \{ g :\xset\to\reals^d |
		\sup_{x\neq y\in \xset} \max(\norm{g(x)}^*,\norm{\grad g(x)}^*, \frac{\norm{\grad g(x) - \grad g(y)}^*}{\norm{x-y}}) \leq 1\}$, the \emph{classical Stein set} of~\cite{GorhamMa15} with arbitrary vector norm $\norm{\cdot}$.
    \item\label{graph-set}
$\gset_n = \gsteinset{}{G} 
		\defeq \{ g \mid 
			\forall\, x\in V,\ \maxarg{\norm{g(x)}^*, \norm{\grad g(x)}^*} \leq 1 
			\text{  and, } \forall\, (x, y) \in E,\\ 
			\max(\textstyle\frac{\norm{g(x) - g(y)}^*}{\norm{x - y}},
				\textstyle\frac{\norm{\grad g(x) - \grad g(y)}^*}{\norm{x - y}},\textstyle\frac{\norm{g(x) - g(y) - {\grad g(x)}{(x - y)}}^*}{\frac{1}{2}\norm{x - y}^2}) \leq 1\}$, the \emph{graph Stein set} of~\cite{GorhamMa15} with arbitrary vector norm $\norm{\cdot}$ and a finite graph $G = (V,E)$ with vertices $V \subset \xset$.  
\end{enumerate}
If $Q_n \not\Rightarrow P$, then either $\langstein{Q_n}{\gset_{b,n}}\not\to 0$ or $(Q_n)_{n=1}^\infty$ is not tight.
\end{theorem}

Next, we prove in \cref{sec:proof-ssds-detect-bounded-sd-non-convergence} that every SSD detects the non-convergence of its bounded SD.
\begin{theorem}[SSDs detect bounded SD non-convergence]
\label{ssds-detect-bounded-sd-non-convergence}
If $\opstein{Q_n}{\gset_{b,n}} \not\to 0$, then, with probability $1$,
$\ssdn{Q_n} \not\to 0$. 
\end{theorem}

Finally, we prove in \cref{sec:proof-ssd-tightness} that SSDs with coercive
(radially unbounded) test functions \emph{enforce tightness}, that is,
remain bounded away from $0$ whenever $(Q_n)_{n=1}^\infty$ is not tight.
\begin{proposition}[Coercive SSDs enforce tightness]
\label{ssd-tightness}
If $(Q_n)_{n=1}^\infty$ is not tight and $\frac{L}{m}\opsub{\sset} g$ is coercive and bounded below for some $g\in
\bigcap_{n= 1}^\infty \gset_n$ and $\forall\sset\in\equisubs{L}{m}$, then surely $\ssdn{Q_n} \not\to 0$.
\end{proposition}

Taken together, these results imply that SSDs equipped with the Langevin operator and any of the convergence-determining Stein sets of \cite{GorhamMa15,GorhamDuVoMa19,GorhamMa17,ChenMaGoFXOa18,ChenBaFXGoGiMaOa19} 
detect non-convergence with probability $1$ under standard dissipativity and growth conditions on the subsampled operator.
\begin{theorem}[Coercive SSDs detect non-convergence]
\label{coercive-ssds-detect-non-convergence}
Under the notation of \cref{bounded-sds-detect-tight-non-convergence}, suppose 
$\grad \log p$ is Lipschitz, $\sup_{x\in\xset}\frac{\twonorm{\grad\log p_{\sset}(x)}}{1 + \twonorm{x}} < \infty$ for all $\sset\in\equisubs{L}{m}$, and, for some $\kappa>0$ and $r\geq 0$,
\balignt\label{eq:sig-drift-growth}
\frac{L}{m}\inner{\grad \log p_\sig(x)-\grad \log p_\sig(y)}{x-y} \leq -\kappa \twonorm{x-y}^2 + r, \ \ \forall x,y\in\reals^d \text{ and } \forall\sset\in\equisubs{L}{m}.
\ealignt
Consider the radial functions $\Phi_1(x) \defeq (1 + \twonorm{x}^2)^{\beta_1}$ for $\beta_1 \in
(-1, 0)$ and $\Phi_2(x) \defeq (\alpha + \log(1 + \twonorm{x}^2))^{\beta_2}$
for $\alpha > 0, \beta_2 < 0$ underlying the
inverse multiquadric and log inverse kernels \cite{ChenMaGoFXOa18} respectively.  For each $n \geq
1$, suppose also that $\gset_n$ satisfies \cref{graph-set},
\cref{classical-set}, or \cref{kernel-set} with kernel $k(x,y) =
\Phi_j(\Gamma (x-y))$ for $j\in\{1,2\}$ and any positive definite
matrix $\Gamma$.  If $Q_n \not\Rightarrow P$, then, with probability $1$,
$\langssdn{Q_n} \not\to 0$.
\end{theorem}
We prove this claim in \cref{sec:proof-coercive-ssds-detect-non-convergence}.

\subsection{Convergence of SVGD and SSVGD}
\label{sec:ssvgd-converges}
Discussing the convergence of SVGD and SSVGD on will require some additional notation. 
For each step size $\eps > 0$ and suitable probability measure $\mu$, define the SVGD update rule
\balignt\label{eq:svgd-transport-map}
T_{\mu,\eps}(x) = x + \eps\Esubarg{X'\sim \mu}{\grad\log p(X')k(X',x) + \grad k(X', x)},
\ealignt
and let $\Phi_\eps(\mu)$ denote the distribution of $T_{\mu,\eps}(X)$ when $X\sim \mu$.
If SVGD is initialized with the point set $(x_{i,0}^n)_{i=1}^n$, then the output of SVGD after each round $r$ is described by the recursion $Q_{n,r} = \Phi_{\eps_{r-1}}(Q_{n,r-1})$ for $r > 0$ with $Q_{n,0} \defeq \frac{1}{n}\sum_{i=1}^n \delta_{x_{i,0}^n}$.  

\citet{liu2017stein} used this recursion to analyze the convergence of non-stochastic SVGD in three steps.
First, Thm.~3.2 of \citep{liu2017stein} showed that, if the SVGD initialization $Q_{n,0}$ converges weakly to a probability measure $Q_{\infty,0}$ as $n \to \infty$, then, on each round $r > 0$, the $n$-point output $Q_{n,r}$ converges weakly to $Q_{\infty,r} \defeq \Phi_{\eps_{r-1}}(Q_{\infty,r-1})$. 
Next, Thm.~3.3 of \citep{liu2017stein} showed that the Langevin KSD $\langstein{Q_{\infty,r}}{\ksteinset{k}} \to 0$ as $r \to\infty$ for a suitable sequence of step sizes $\eps_r$.
Finally, Thm.~8 of \citep{GorhamMa17} implied that $Q_{\infty,r} \Rightarrow P$ for suitable kernels and targets $P$.

A gap in this analysis lies in the stringent assumptions of the first step: Thm.~3.2 of \citep{liu2017stein} only applies when $f(x,z) = \grad \log p(x) k(x,z) + \grad_x k(x,z)$ is both bounded and Lipschitz, but the growth of $\grad\log p(x)k(x,z)$ typically invalidates both assumptions\footnote{Consider, for example, the standard Gaussian $\grad \log p(x) = -x$ with any translation invariant kernel $k$ on $\reals^d$.}.
Indeed, \citet{liu2017stein} remarks,
``Therefore, the condition ... suggests that it can only be used when [the domain] is compact. It is an open question to establish results that can work for more general domain[s].''
Our next theorem, proved in \cref{sec:proof-ssvgd-converges}, 
achieves this goal for $\reals^d$  
by showing that, on round $r$, both the SVGD output $Q_{n,r}$ and the SSVGD output $Q_{n,r}^m$ of \cref{alg:ssvgd} converge to $Q_{\infty,r}$ with probability $1$ under assumptions commonly satisfied by $p$ and $k$.

\begin{theorem}[Wasserstein convergence of SVGD and SSVGD]
\label{ssvgd-converges}
Suppose SVGD and SSVGD are initialized with $Q_{n,0} = \frac{1}{n}\sum_{i=1}^n \delta_{x_{i,0}^n}$
satisfying $\lswass{1}(Q_{n,0}, Q_{\infty,0}) \to 0$.
If for some $c_1,c_2 > 0$,
\balignt\label{eq:M1-lin-growth}
\Lip(\grad \log p(x)k(x,\cdot) + \grad_x k(x,\cdot)) &\leq c_1(1+\twonorm{x}) \qtext{and}\\
\Lip(\grad \log p(\cdot)k(\cdot,z) + \grad_x k(\cdot,z)) &\leq c_2(1+\twonorm{z}),
\ealignt
then $\lswass{1}(Q_{n,r}, Q_{\infty,r}) \to 0$ as $n\to\infty$ for each round $r$.
If, in addition, for some $c_0 > 0$,
\balignt\label{eq:M0-lin-growth}
&\max_{\sset\in\equisubs{L}{m}}\sup_{z\in\reals^d}\twonorm{\grad \log p_{\sset}(x)k(x,z)}
    \leq c_0(1+\twonorm{x}), \\
&\max_{\sset\in\equisubs{L}{m}}\sup_{z\in\reals^d}\opnorm{\grad_x (\grad \log p_{\sset}(x)k(x,z))}
    \text{ is bounded on compact sets } K,
\ealignt
then, for each round $r$, $\lswass{1}(Q_{n,r}^m, Q_{n,r}) \toas 0$ as $n\to\infty$.
\end{theorem}
To illustrate the applicability of \cref{ssvgd-converges}, we highlight that the growth assumptions \cref{eq:M1-lin-growth,eq:M0-lin-growth}  hold for standard bounded radial kernels like the Gaussian, \Matern, inverse multiquadric, and inverse log \citep{ChenMaGoFXOa18} kernels paired with any  Lipschitz $\grad \log p$ and linear-growth $\grad \log p_{\sset}$.
\section{Experiments} \label{sec:experiments}
In this section, we demonstrate the practical benefits of using 
SSDs as drop in replacements for standard SDs. In each of our experiments, the target is a posterior distribution of the form $p(x) \propto \prod_{l=1}^L p_l(x)$ where $p_l(x) \defeq \pi_0(x)^{1/L}\pi(y_l | x)$ for $\pi_0$ a prior density, $\pi(\cdot | x)$  a  likelihood function, and $(y_l)_{l=1}^L$ a sequence of observed datapoints.
The SKSDs in 
\cref{subsec:hyperparameter-selection,subsec:sampler-selection} use an inverse multiquadric base kernel $k(x,y) =
(1 + \twonorm{x-y}^2)^{\beta}$ with $\beta=-\textfrac{1}{2}$ as in \cite{GorhamMa17}.
Julia \cite{Bezanson2014julia} code recreating the experiments in
\cref{subsec:hyperparameter-selection,subsec:sampler-selection} and Python
code recreating the experiments in \cref{subsec:stochastic-svgd} is
available at \url{https://github.com/jgorham/stochastic_stein_discrepancy}.

\subsection{Hyperparameter selection for approximate MCMC} \label{subsec:hyperparameter-selection}

\emph{Stochastic gradient Langevin dynamics} (SGLD) \citep{WellingTe11} with constant step size $\epsilon$ is an approximate MCMC method introduced as a scalable alternative to the popular Metropolis-adjusted Langevin algorithm \citep{RobertsTw96}.
A first step in using SGLD is selecting an appropriate step size $\epsilon$, as overly large values lead to severe distributional biases (see the right panel of the \cref{fig:compare-hyperparameters-gmm-posterior_stochastic}  triptych), while overly small values yield slow mixing (as in the left panel of the \cref{fig:compare-hyperparameters-gmm-posterior_stochastic} triptych). In \cite[Sec. 5.1]{WellingTe11}, the posterior over the
means of a Gaussian mixture model (GMM) was used to illustrate
the utility of SGLD, and in \cite[Sec. 5.3]{GorhamMa15}, the spanner graph
Stein discrepancy was employed to select an appropriate $\epsilon > 0$ for a
fixed computational budget. We recreate the experimental setup of \cite[Sec. 5.3]{GorhamMa15} to assess the ability of a stochastic KSD to effectively tune SGLD.

We used the same model parameterization as \citet{WellingTe11},
which was a posterior distribution with $L = 100$ likelihood terms contributing
to the posterior density. We adopted the same experimental methodology as
\cite[Sec. 5.3]{GorhamMa15}: for a range of $\epsilon$ values, we generated $50$ independent SGLD pilot chains of length $n=1000$. For each sample of size $n$, we computed the
IMQ KSD without any subsampling and the SKSD with batch sizes $m = 1$ and $m = 10$. In Figure
\ref{fig:compare-hyperparameters-gmm-posterior_stochastic}, we see that both
 SKSDs behave in step with the standard KSD: the choice of $\epsilon
= 5\times 10^{-3}$ minimizes the KSD over the average of the $50$ trials
for all variants of KSD. Moreover, the fastest  SKSD required
one hundredth the number of likelihood evaluations of the standard KSD. Hence, subsampling can lead to significant speed-ups with little degradation in inferential quality even when the total number of likelihood terms is moderate.

\begin{figure*}
  \centering
  \includegraphics[width=\textwidth]{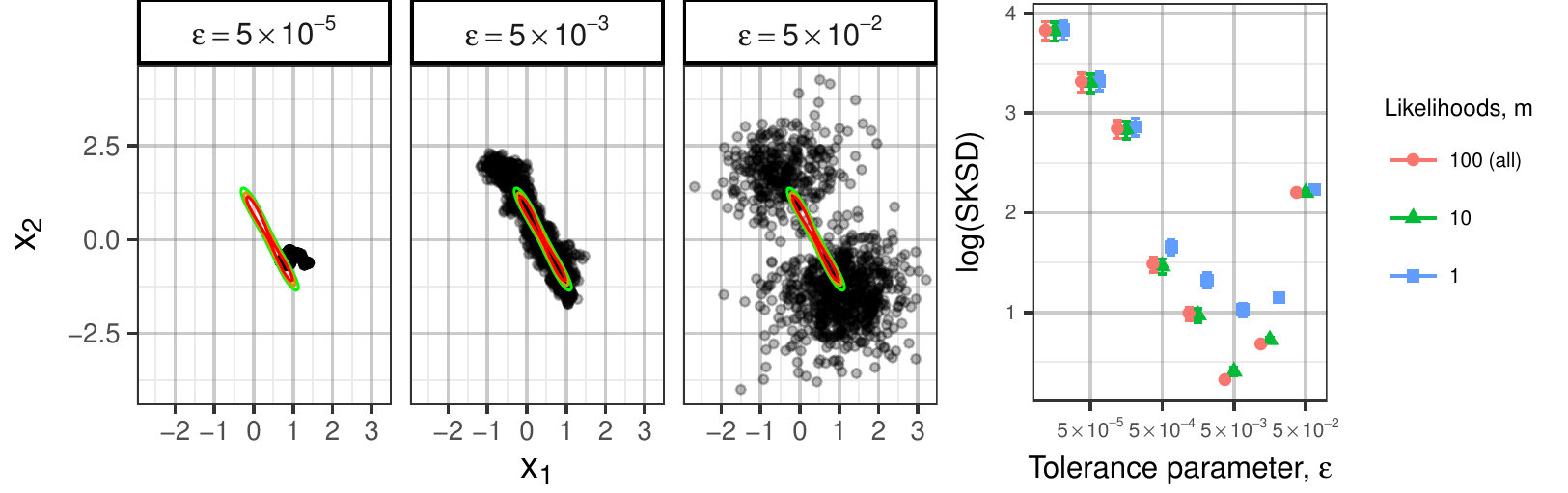}
  \caption{
  \textbf{Left:} Representative samples with $n=1,000$ points obtained
  from SGLD with varying step sizes
  $\epsilon$. The contours represent high density regions of the bimodal
  posterior distribution. Notice the leftmost plot suffers from a lack of
  mixing, while the rightmost plot is far too overdispersed to fit the
  posterior.
  \textbf{Right:} For different subsampling sizes $m$ of the $L=100$ likelihood
  terms contributing to the posterior, the mean
  IMQ SKSD ($\pm 1$ standard error) over $50$ trials for each choice of
  $\epsilon$ is shown on a log scale. Even at extreme subsampling rates,
  the SKSD produces the same ranking of candidates and selects the same
  $\epsilon$ as the exact KSD.
  }
  \label{fig:compare-hyperparameters-gmm-posterior_stochastic}
\end{figure*}

\subsection{Selecting biased MCMC samplers} \label{subsec:sampler-selection}

\citet[Sec.\ 4.4]{GorhamMa17} used the KSD
to choose between two biased sampling procedures. Namely, they
compared two variants of the approximate MCMC algorithm \emph{stochastic gradient Fisher scoring} (SGFS)
\citep{Ahn2012}. The full variant of this sampler---called
SGFS-f---requires inverting a $d\times d$ matrix to produce each sample
iterate. A more computationally expedient variant---called SGFS-d---instead
inverts that $d\times d$ matrix but first zeroes out all off-diagonal entries.
Both MCMC samplers are uncorrected discretizations of a continuous-time
process, and their invariant measures are asymptotically biased away from the
target $P$. Accordingly, the SSD can be employed to assess whether the greater
number of sample iterates generated by SGFS-d under a fixed computational
budget outweighs the additional cost from asymptotic bias.

In both \citep[Sec 4.4]{GorhamMa17} and \citep[Sec 5.1]{Ahn2012}, the chosen
target $P$ was a Bayesian logistic regression with a flat prior. The training set
was constructed by selecting a subset of $10,000$ images from the MNIST
dataset that had a $7$ or $9$ label, and then reducing each covariate vector of
$784$ pixel values to a dimension $50$ vector via random projections. After
including an intercept term, \citet{Ahn2012} generated a
posterior sample of $50,000$ sample iterates (each in $\reals^{51}$) for both samplers.
In \citep[Sec 4.4]{GorhamMa17}, the authors showed the KSD
preferred the sample iterates generated from SGFS-f for any number of
sample iterates, while in \citep[Sec 5.1]{Ahn2012}, the authors showed even the
best bivariate marginals generated by SGFS-d were inferior to SGFS-f at matching
the target posterior $P$.

\begin{figure*}
  \centering
    \includegraphics[width=\textwidth]{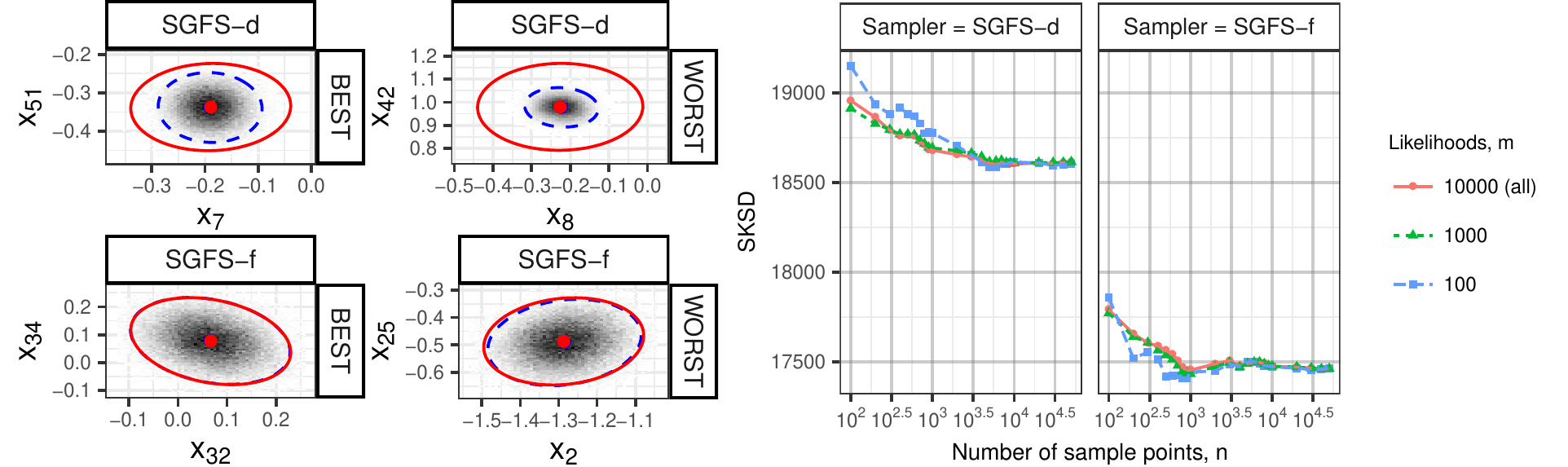}
  \caption{
  \textbf{Left:}
  The best and worst bivariate density plots of $50,000$ SGFS sample
  iterates that approximate the true target posterior distribution
  $P$. The overlaid lines are the bivariate marginal means and 95\%
  confidence ellipses; the dashed blue are derived from SGFS samples and the
  solid red is derived from a surrogate ground truth sample.
  \textbf{Right:} Plot of exact KSD (red) and stochastic KSDs (green and
  blue) for each SGFS sampler vs. the number of sample iterates $n$.
  }
  \label{fig:mnist_7_or_9_sgfs_stochastic}
\end{figure*}

In Figure \ref{fig:mnist_7_or_9_sgfs_stochastic}, we compare the exact KSDs
with the stochastic KSDs obtained from sampling $100$ and $1,000$ of the
$10,000$ likelihood terms i.i.d. for each posterior sample iterate. Notice
that the stochastic KSD prefers SGFS-f over SGFS-d for each subsampling
parameter as well, in accordance with the exact KSD. However, the most
aggressively subsampled stochastic KSD requires $100$ times fewer
likelihood evaluations than its standard analogue.

\subsection{Particle-based variational inference with SSVGD}
\label{subsec:stochastic-svgd}
SVGD was developed to iteratively improve a $n$-point particle approximation $Q_n$ to a given target distribution.
To illustrate the practical benefit of the stochastic SVGD algorithm analyzed in \cref{sec:ssvgd-converges} over standard SVGD, we reproduce the Bayesian neural network
experiment from \cite[Sec. 5]{LiuWa2016stein} on three datasets used in their experiment.
We adopt the exact experimental setup of \cite{LiuWa2016stein} and adapt their code to compare SSVGD (\cref{alg:ssvgd}) with minibatch sizes $m = 0.1 L$ and $m = 0.25 L$ with standard SVGD ($m = L$).
The procedure was run $20$ times for each configuration, and each time
we started with an independently sampled train-test split.
The training sets for the \texttt{boston}, \texttt{yacht}, and \texttt{naval} datasets had $409$, $209$, and $10,241$ datapoints and
$d=13, 6$, and $17$ covariates, respectively. The \texttt{boston} dataset was first published in
\cite{HaRu1978} while the latter two are available on the UCI repository \cite{Dua19}.
The root mean-squared error (RMSE) and log likelihood were computed on the
test set, and a summary is presented in \cref{fig:stochastic-svgd}. 
SSVGD yields more accurate approximations for all likelihood computation budgets considered, even for the modestly sized datasets, and this effect is magnified in the larger \texttt{naval} dataset.

\begin{figure*}
  \centering
    \includegraphics[width=\textwidth]{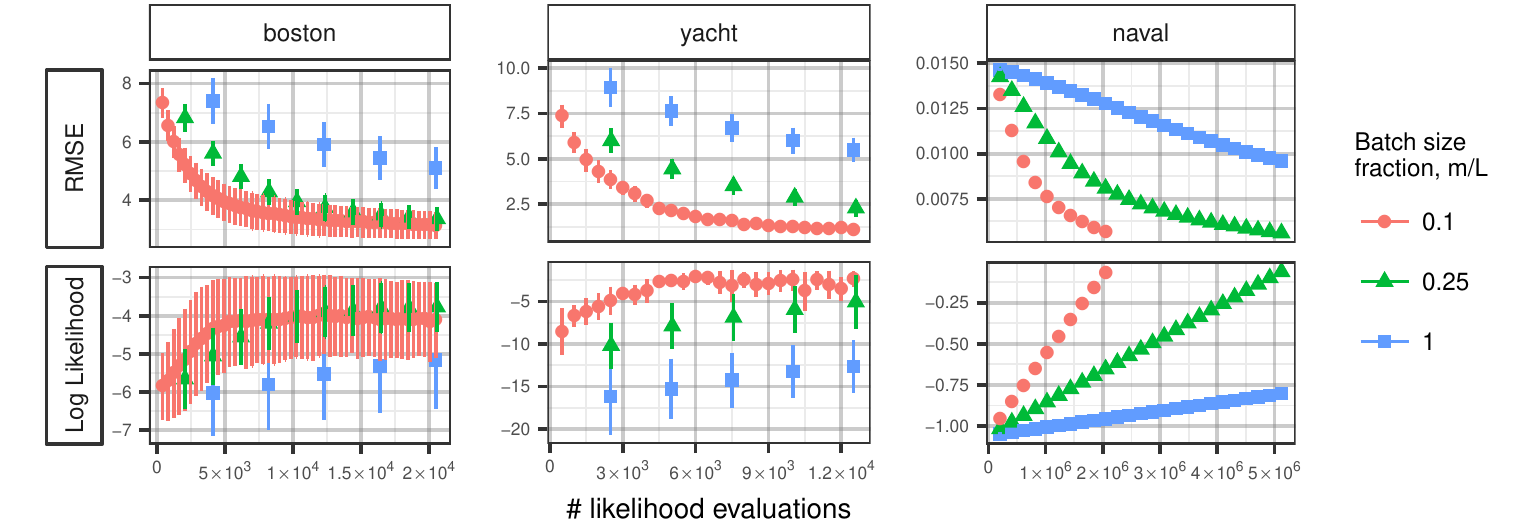}
  \caption{We plot the test RMSE and log likelihood
    for a Bayesian neural net as we approximate the posterior using
    stochastic SVGD over a set of sampling rates. The results shown above are the
    mean RMSE and log likelihood with $\pm 2$ standard errors of that mean across $20$ train test
    splits for the boston, yacht and naval datasets.
    We see that for each likelihood budget considered
    the lower batch sizes produce more accurate approximations than full batch SVGD.
  }
  \label{fig:stochastic-svgd}
\end{figure*}

\section{Discussion and Future Work}\label{sec:conclusion}
To reduce the cost of assessing and improving sample quality, we introduced stochastic Stein discrepancies which inherit the convergence-determining properties of standard SDs with probability $1$ while requiring orders of magnitude fewer likelihood evaluations.
While our work was focused on measuring sample quality, we believe that other inferential tasks based on decomposable Stein operators can benefit from these developments. 
Prime candidates include SD-based goodness-of-fit testing \cite{ChwialkowskiStGr2016,LiuLeJo16,jitkrittum2017linear,HugginsMa2018}, KSD-based sampling
\cite{ChenMaGoFXOa18,futami2019bayesian,ChenBaFXGoGiMaOa19}, improving Monte Carlo estimation with control variates \cite{AssarafCa1999,MiraSoIm2013,OatesGiCh2016},  improving sample quality through reweighting \cite{LiuLe2016,hodgkinson2020reproducing} or thinning \cite{RiabizChCoSwNiMaOa2020}, and parameter estimation in intractable models \cite{BarpFXDuGiMa2019}.
Integrating variance reduction techniques \cite[e.g.,][]{SchmidtRoBa17,DeFazioBaLa14} into the SSD computation is another promising direction, as the result could more closely mimic standard SDs while offering comparable computational savings.
Finally, while the Langevin operator received special attention in our
analysis, our results also extend to other popular Stein operators like the
diffusion operators of \cite{GorhamDuVoMa19} and the discrete operators of
\cite{yang2018goodness}.

\section*{Broader Impact}
This work provides both producers and consumers of approximate inference techniques with a valid diagnostic for assessing those approximations at scale. 
It also analyzes a scalable algorithm (SSVGD) for improving approximate inference.
We expect that many existing users of Stein discrepancies will want to employ stochastic Stein discrepancies to reduce their overall computational costs.
In addition, the ready availability of a scalable diagnostic may stimulate the more widespread use of approximate MCMC methods.
However, any inferential tool combined with the wrong data or inappropriate
model can lead to incorrect and harmful conclusions, so care must be taken
in interpreting the results of any downstream analysis.

\begin{ack}
We thank Sungjin Ahn, Anoop Korattikara, and Max Welling for sharing their
MNIST posterior samples.
Part of this work was completed while Anant Raj was an intern at Microsoft Research.

\end{ack}

\bibliographystyle{abbrvnat}
\bibliography{stein}

\newpage
\appendix
\section{Proof of \cref{sksd-formula}: SKSD closed form}
\label{sec:proof-sksd-formula}
Our proof will parallel that of \citet[Prop.~2]{GorhamMa17} for non-stochastic KSDs.
For each $j\in[d]$ and each $\sigma_i$, we define the coordinate operators
\balignt
\frac{L}{m}(\opsub{\sig_i}^jf)(x) 
	\defeq (\frac{L}{m} \grad_{x_j}\log p_{\sig_i}(x)+\grad_{x_j}) f(x)
\ealignt for $f:\reals^d\to\reals$. 
For each $g=(g_1,\dots,g_d)\in\ksteinsetnorm{k}{\norm{\cdot}}$ and $x\in\reals^d$, 
our $C^{(1,1)}$ assumption on $k$ and the proof of \citep[Cor. 4.36]{Christmann2008}
imply that
\balignt\label{eqn:langevin-componentwise}
  \opsubarg{\sigma_i}{g}{x}
    &\textstyle= \sum_{j=1}^d (\opsub{\sigma_i}^j{g_j})(x)
    = \sum_{j=1}^d \opsub{\sigma_i}^j{\inner{g_j}{k(x,\cdot)}_{\kset_k}}
    = \sum_{j=1}^d \inner{g_j}{\opsub{\sigma_i}^jk(x,\cdot)}_{\kset_k}.
\ealignt
Meanwhile, the result \citep[Lem. 4.34]{Christmann2008} yields
\balignt\label{eqn:kj-derivation}
\inner{\frac{L}{m}\opsub{\sigma_i}^jk(x_i,\cdot)}{\frac{L}{m}\opsub{\sigma_{i'}}^jk(x_{i'},\cdot)}
    &= (\frac{L}{m}\grad_{x_{ij}}\log p_{\sig_i}(x_i) + \grad_{x_{ij}})(\frac{L}{m}\grad_{x_{i'j}}\log p_{\sig_{i'}}(x_{i'}) +\grad_{x_{i'j}})k(x_i,x_{i'})
\ealignt
for all $i, i' \in[n]$ and $j\in[d]$.
Therefore, the advertised
\balignt\label{eqn:kj-closed-form}
w_j^2
   = \frac{1}{n^2}\sum_{i=1}^n\sum_{i'=1}^n \inner{\frac{L}{m}\opsub{\sigma_i}^jk(x_i,\cdot)}{\frac{L}{m}\opsub{\sigma_{i'}}^jk(x_{i'},\cdot)}
   = \knorm{\frac{1}{n}\sum_{i=1}^n{\frac{L}{m}\opsub{\sigma_i}^j{k(x_i,\cdot)}}}^2.
\ealignt
Finally, our assembled results and norm duality give
\balignt
\sksd{Q_n}
  &= \sup_{g\in\ksteinsetnorm{k}{\norm{\cdot}}}
  \sum_{j=1}^d\frac{1}{n}\sum_{i=1}^n \frac{L}{m}(\opsub{\sig_i}^jg_j)(x_i) \\
  &= \sup_{\knorm{g_j} = v_j, \norm{v}^* \leq 1} \textsum_{j=1}^d
     \inner{g_j}{\frac{1}{n}\sum_{i=1}^n \frac{L}{m}\opsub{\sigma_i}^j{k(x_i,\cdot)}}_{\kset_k} \\
  &= \sup_{\norm{v}^* \leq 1} \textsum_{j=1}^d v_j
    \knorm{\frac{1}{n}\sum_{i=1}^n \frac{L}{m}[\opsub{\sigma_i}^j{k(x_i,\cdot)}}\\
  &= \sup_{\norm{v}^*\le 1} \textsum_{j=1}^d v_j w_j
  = \norm{w}.
\ealignt

\section{Proof of \cref{ssds-detect-convergence}: SSDs detect convergence}
\label{sec:proof-ssds-detect-convergence}
We will find it useful to write
\begin{talign}\label{eqn:ssd-definition}
    \ssd{Q_n} 
        &= \sup_{g\in\gset} \left |\frac{1}{n}\sum_{i=1}^n
    \frac{L}{m}\sum_{\sset\in\equisubs{L}{m}} B_{i\sset}
    \opsubarg{\sset}{g}{x_i}\right | \qtext{for} B_{i\sset} \defeq \indic{\sset = \sset_i}\\
        &= \sup_{g\in\gset}\left |{L\choose m}^{-1}\sum_{\sset\in\equisubs{L}{m}} \mu_{n\sset} (\opsub{\sset}{g})\right|
        \qtext{for} \mu_{n\sset} \defeq {L \choose m}\frac{L}{m}\frac{1}{n}\sum_{i=1}^n B_{i\sset} \delta_{x_i}.
\end{talign}
We will also write $\blset \defeq \{h:\reals^d\to\reals : \infnorm{h} +
\Lip(h) \le 1\}$ as the unit ball in the bounded Lipschitz metric, and for
any $R > 0$, $B_R \defeq \{x\in\reals^d : \twonorm{x}\le R\}$ as the radius
$R$ ball centered at the origin.
For any set $K$, let $I_K(x) = \indic{x \in K}$.

Our proof relies on a lemma, proved in
\cref{sec:proof-random_measure_convergence}, that 
boosts almost sure convergence in distribution into almost sure uniform
convergence for the expectations of all continuous functions dominated by a uniformly integrable, locally bounded
$|f_0|$ with derivatives dominated by a locally bounded $|f_1|$.

\begin{lemma}[Convergence of random measures]\label{random_measure_convergence}
Consider two sequences of random measures $(\nu_n)_{n=1}^\infty$ and
$(\tilde{\nu}_n)_{n=1}^\infty$ on $\reals^d$, and
suppose there exists an $R > 0$ such that $\nu_n(hI_{B_R}) -
\tilde{\nu}_n(hI_{B_R}) \toas 0$ for each bounded and continuous $h$.
Then, for $\hset = \blset$,
\begin{align}\label{eq:uniform_compact_continuous_convergence}
\sup_{h \in \hset}
|\nu_{n}(h I_{B_R}) - \tilde{\nu}_n(h  I_{B_R})| \toas 0.
\end{align}
Suppose, in addition, that for every $S > 0$ there exists an $R \geq S$ such that
\cref{eq:uniform_compact_continuous_convergence} holds. Then if
$f_0$ is almost surely uniformly $\nu_n$-integrable and
uniformly $\tilde{\nu}_n$-integrable, and $f_0, f_1$ are bounded on each
compact set, we have
\begin{align}
\sup_{h \in \hset_f}
|\nu_{n}(h) - \tilde{\nu}_n(h)| \toas 0,
\end{align}
where $\hset_f \defeq \{h\in C(\reals^d) : |h(x)|\le |f_0(x)| , \frac{|h(x) -
  h(y)|}{\twonorm{x-y}} \le |f_1(x)| + |f_1(y)| \text{ for all } x,y\in\reals^d\}$.
\end{lemma}

Since $\lswass{a}(Q_n, P) \to 0$, \citep[Proof~of~Cor.~1]{ekisheva2006transportation} implies that $Q_n(h) \to P(h)$ for all bounded continuous $h$ and that $f_0(x) = c(1+\twonorm{x}^a)$ is uniformly $Q_n$-integrable and $P$-integrable.
Moreover, for each $\sset \in \equisubs{L}{m}$, $\mu_{n\sset}(h) -
\frac{L}{m}Q_n(h) \toas 0$ for all bounded $h$ by
\cref{bounded_func_convergence}, and thus $\mu_{n\sset}(h I_{B_R}) - Q_n(h
I_{B_R}) \toas 0$ for all bounded $h\in C(\reals^d)$ and any $R > 0$.
Since, for any compact set $K$, $\mu_{n\sset}(|f_0|I_{K^c}) \leq
{L\choose m}\textfrac{L}{m} Q_n(|f_0|I_{K^c})$, $f_0$ is also uniformly
$\mu_{n\sset}$-integrable. By assumption $f_1(x) = \omega(\twonorm{x})$
for $\omega(R) \defeq \sup_{n}\sup_{g \in \gset_n, x,y\in B_{2R}}
\textfrac{|\opsubarg{\sset}{g}{x} - \opsubarg{\sset}{g}{y}|}{\twonorm{x-y}}$
is bounded on any compact set.

Moreover, since $P$ is a finite measure, there are at most countably many
values $R$ for which $P(\{x : \twonorm{x} = R\}) > 0$.
Hence, for any $S >
0$ we can choose $R \geq S$ such that $B_R$ is a continuity set under
$P$. 
For any such $R$, $Q_n(h I_{B_R}) - P(h I_{B_R})\to 0$ for any bounded
$h\in C(\reals^d)$ by the Portmanteau theorem \citep[Thm. 13.16]{Klenke13},
since $\lswass{a}(Q_n, P) \to 0$ implies convergence in distribution. 

Finally, the assumption $P(\operator{g}) = 0$ for all
$g\in\gset_n$, the triangle inequality, the continuity and polynomial growth
of each function in $\opsub{\sset}{\gset_n}$, and
\cref{random_measure_convergence} applied first to $\mu_{n\sset}$ and $(Q_n)_{n=1}^\infty$ for each $\sigma$ and then to $(Q_n)_{n=1}^\infty$ and $P$ together yield
\balignt
&\ssdn{Q_n}
    =\sup_{g\in\gset_n} |{L\choose m}^{-1}\sum_{\sset\in\equisubs{L}{m}}\mu_{n\sset}(\opsub{\sset}{g}) - \frac{L}{m}Q_n(\opsub{\sset}{g}) + \frac{L}{m}Q_n(\opsub{\sset}{g}) - \frac{L}{m}P(\opsub{\sset}{g})| \\
    &\qquad\leq  {L\choose m}^{-1}\sum_{\sset\in\equisubs{L}{m}} \sup_{h \in \hset_f}|\mu_{n\sset}(h) - \frac{L}{m}Q_n(h)|
    + \frac{L}{m}|Q_n(h) - P(h)|
    \toas 0.
\ealignt

\subsection{Proof of \cref{random_measure_convergence}: Convergence of random measures}
\label{sec:proof-random_measure_convergence}
Fix any $R, \eps > 0$ and let $K=B_R$. By the Arzelà–Ascoli theorem
\citep[Thm. 8.10.6]{Durrett19}, there exists a finite $\eps/2$-subcover
of the set of $K$-restrictions $\{h|_K : h\in\hset\}$.  Since any bounded
continuous function on $K$ can be extended to a bounded continuous function
on $\reals^d$, there therefore exists a sequence of bounded continuous
functions $(h_k)_{k=1}^m$ on $\reals^d$ such that
\balignt
    \P(\sup_{h \in \hset} |\nu_{n}(h I_K) - \tilde{\nu}_n(h I_K)| > \eps \io)
&\leq
    \P(\max_{1\leq k\leq m} |\nu_{n}(h_k I_K) - \tilde{\nu}_n(h_k I_K)| > \eps/2 \io) \\
&\leq 
    \sum_{k=1}^m \P(|\nu_{n}(h_k) - \tilde{\nu}_n(h_k)| > \eps/2 \io)
    = 0,
\ealignt
where we have used the union bound and our almost sure convergence assumption for bounded continuous functions.
The first result \cref{eq:uniform_compact_continuous_convergence} now follows since $\eps$ was arbitrary.

We next assume that the event $\event$ on which $f_0$ is uniformly $\nu_n$
and $\tilde{\nu}_n$-integrable occurs with probability $1$, and fix any
$\eps > 0$. On $\event$ there exists $R_{\eps} > 0$ such that
\cref{eq:uniform_compact_continuous_convergence} holds and
$\sup_n \max(\nu_n(|f_0|I_{K_\eps^c}),
\tilde{\nu}_n(|f_0|I_{K_\eps^c})) \leq \eps/2$ for $K_{\eps} \defeq
B_{R_\eps}$.
Furthermore, on $\event$,
\balignt
    \sup_{h\in\hset_f}|\nu_{n}(h) - \nu_{n}(hI_{K_\eps})| + |\tilde{\nu}_{n}(h) - \tilde{\nu}_{n}(hI_{K_\eps})|
&\leq \sup_{h\in\hset_f} \nu_{n}(|h|I_{K_\eps^c}) + \tilde{\nu}_{n}(|h|I_{K_\eps^c}) \\
&\leq \nu_{n}(|f_0|I_{K_\eps^c}) + \tilde{\nu}_{n}(|f_0|I_{K_\eps^c}) \leq \eps.
\ealignt
Therefore, the triangle inequality, fact that for each $R > 0$ there is a
constant $c_R > 0$ such that $\{hI_{B_R} : h\in \hset_f\}\subseteq \{c_R hI_{B_R}:h\in\hset\}$,
and our first result \cref{eq:uniform_compact_continuous_convergence} give
\balignt
    \P(\sup_{h \in\hset_f} |\nu_{n}(h) - \tilde{\nu}_n(h)| > 2\eps \io)
    &\leq\P(\event^c) + \P(\sup_{h \in \hset_f} |\nu_{n}(hI_{K_\eps}) -
    \tilde{\nu}_n(hI_{K_\eps})| > \eps \io) \\
    &\leq\P(\event^c) + \P(c_{R_\eps}\sup_{h \in \hset} |\nu_{n}(hI_{K_\eps}) -
    \tilde{\nu}_n(hI_{K_\eps})| > \eps \io) \\
    &=0.
\ealignt
The second result now follows since $\eps$ was arbitrary.

\section{Proof of \cref{bounded-sds-detect-tight-non-convergence}: Bounded SDs detect tight non-convergence}
\label{sec:proof-bounded-sds-detect-tight-non-convergence}

We consider each Stein set candidate in turn.
\subsection{Kernel Stein set}
Suppose $\gset_n$ satisfies \cref{kernel-set}.  Since, for any vector norm
$\norm{\cdot}$ on $\reals^d$, there exists $c_d$ such that $\{ g \in
\ksteinsetnorm{k}{\twonorm{\cdot}}: \max_{\sset\in\equisubs{L}{m}}\supnorm{\opsub{\sset}g}\leq 1\} \subseteq
c_d \{ g \in \ksteinsetnorm{k}{\norm{\cdot}}: \max_{\sset\in\equisubs{L}{m}}\supnorm{\opsub{\sset}g}\leq
1\}$ \citep{Bachman1966functional}, it suffices to assume $\norm{\cdot} =
\twonorm{\cdot}$.

\paragraph{Choosing a convergence-determining IPM $d_\hset$}
Consider the test function set $\hset$ from 
\cite[Sec E.1, Proof of Thm. 5]{GorhamMa17} which satisfies \begin{enumerate}
    \item $\| h \|_{\infty}\le 1$ and $\Lip(h)\le 1 + \sqrt{d - 1}$ for all $h\in\hset$ and
    \item $Q_n\not\Rightarrow P$ implies $d_{\hset}(Q_n, P)\not\to 0$ for any sequence of probability measures $(Q_n)_{n\geq 1}$.
\end{enumerate}

\paragraph{Solving the Stein equation $\langevin{g_h} = h - P(h)$}
\newcommand{\MP}{\mathcal{M}_P}
Let us define $\Xi(x) \defeq (1 + \twonorm{x}^2)^{1/2}$. 
By \cite[Sec E.1, Proof of Thm. 5]{GorhamMa17}, for each $h\in\hset$ there exists an
accompanying function $g_h$ such that $\langevin{g_h} = h - P(h)$ and $\|\Xi g_h\|_{\infty}\le \MP$ for a constant $\MP > 0$ independent of $h$.

\paragraph{Smoothing the Stein function $g_h$}
Fix any $\rho \in (0,1]$, and let $U\sim \Gsn(0,I)$.
Since $\grad \log p$ is Lipschitz,
the argument in \cite[Proof of Thm. 13]{GorhamMa17} constructs a smoothed approximation $g_{h,\rho}(x) = \E[g_h(x-\rho U)]$ satisfying
\begin{align}
\supnorm{\langevin{g_{h,\rho}} - \langevin{g_h}} \leq C_1 \rho \label{eq:rho-approx-error}
\end{align}
for a constant $C_1$ independent of $h$ and $\rho$.
Moreover, the following lemma shows that 
\begin{align}
\|\Xi g_{h,\rho}\|_{\infty}
    \leq \supnorm{\Xi g_h} \sqrt{2} \E[1+\twonorm{U}]
    \le \MP' \defeq \sqrt{2} \MP(1+\sqrt{d}),
\end{align}
where $\MP$ is notably independent of $\rho$ and $h$.
\begin{lemma}[Smoothing preserves decay]
\label{smoothing-preserves-decay}
For each $g : \reals^d \to \reals^d$, $\eps \in [0,1]$, and absolutely integrable random vector $Y \in \reals^d$,
\balignt
\sup_{x\in\reals^d} \Earg{A(x)\twonorm{g(x - \eps Y)}}
  \le \sqrt{2}\supnorm{\Xi g} \Earg{A(Y)}
  \qtext{for}
  A(x) \defeq 1+\twonorm{x}.
  \label{eq:tilted-approx-stein-bound}
\ealignt
\end{lemma}
\begin{proof}
For $B(y) \defeq \sup_{x,u\in(0, 1]} A(x) / \Xi(x - u y)$, we have
\balignt
\sup_{x\in\reals^d} \Earg{(1+\twonorm{x})\twonorm{g(x - \eps Y)}}
  &= \sup_{x\in\reals^d} \Earg{\textfrac{(1+\twonorm{x})}{\Xi(x - \eps Y)}\Xi(x - \eps Y)\twonorm{g(x -
      \eps Y)}} \\
  &\le \sup_{x\in\reals^d} \supnorm{\Xi g}\Earg{\textfrac{(1+\twonorm{x})}{\Xi(x - \eps Y)}}
  \le \supnorm{\Xi g} \Earg{B(Y)}.
\ealignt
Moreover, $\Xi(z)\ge 2^{-1/2}(1 + \twonorm{z})$ for
  all $z$ implies that, for any $y$,
\balignt
B(y) = \sup_{x,u\in(0,1]}\frac{A(x)}{\Xi(x-uy)}
  &\le \sup_{x,u\in(0,1]}\sqrt{2}\frac{A(x)}{1 + \twonorm{x-uy}}
  = \sup_{z,u\in(0,1]}\sqrt{2}\frac{A(z+uy)}{1 + \twonorm{z}} \\
  &\le \sup_{z,u\in(0,1]}\sqrt{2}\frac{A(z) + u\twonorm{y}}{1 +
      \twonorm{z}}
  \le \sqrt{2}A(y),
  \ealignt
where we used the triangle inequality in the penultimate inequality.
\end{proof}

\paragraph{Truncating the smoothed Stein function $g_{h,\rho}$}
Fix any $\eps \in (0,1)$, and, since $(Q_n)_{n=1}^\infty$ is tight, select a
compact set $K_{\eps}$ satisfying $\sup_n Q_n(K_\eps^c) \leq \eps$. 
The argument in \cite[Proof of Thm. 13]{GorhamMa17} identifies a truncation $g_{h,\rho,\eps}$ and a constant $C_0$ independent of $h$, $\eps$, and $\rho \in (0,1]$ such that, for all $x\in\reals^d$,
\begin{align}
&\twonorm{g_{h,\rho,\eps}(x)}
    \le \twonorm{g_{h,\rho}(x)} \qtext{and} \\
&|\langarg{g_{h,\rho,\eps}}{x} - \langarg{g_{h,\rho}}{x}|
    \le C_0 \indic{x \in K_\eps^c}.\label{eq:g0-approx-error}
\end{align} 
Hence, $\|\Xi g_{h,\rho,\eps}\|_{\infty}\le \|\Xi g_{h,\rho}\|_{\infty} \le \MP'$.

\paragraph{Smoothing the truncation $g_{h,\rho,\eps}$}
By assumption, for all $\sset \in \equisubs{L}{m}$, there is a constant
$\beta > 0$ such that $\twonorm{\grad\log p_{\sset}(x)} \le
\beta (1 + \twonorm{x})$ for all $x$. Defining
$A_\beta(x) \defeq \textfrac{L}{m}\beta (1 +
\twonorm{x})$, we note that, since $\grad\log p =
\textfrac{L}{m}{L\choose m}^{-1}\sum_{\sset\in\equisubs{L}{m}} \grad\log
p_{\sset}$, an application of the triangle inequality yields
$\twonorm{\grad\log p(x)} \le A_\beta(x)$ for all $x$. Moreover, since $L/m \ge 1$
we have $\twonorm{\grad\log p_{\sset}(x)}\le A_\beta(x)$ for all $x$ and $\sset$.

From the construction in \cite[Proof of Lem. 12]{GorhamMa17}, there is a
random variable $Y$ with finite first moment such that the function 
$\tilde{g}_{h,\rho,\eps}(x)\defeq
\Earg{g_{h,\rho,\eps}(x - \eps Y)}$ satisfies 
\begin{align}
\|\langevin{\tilde{g}_{h,\rho,\eps}} - \langevin{g_{h,\rho,\eps}}\|_{\infty}
    \le C_\rho \eps \label{eq:geps-approx-error}
\end{align}
 and $\tilde{g}_{h,\rho,\eps} \in C_{\eps,\rho}\gset_n$ for constants $C_\rho$
independent of $\eps$ and $h$ and $C_{\eps,\rho}$ independent of $h$.

\paragraph{Showing the smoothed truncation $\tilde{g}_{h,\rho,\eps}$ is in a scaled copy of $\gset_{b,n}$}
By \cref{smoothing-preserves-decay}, we have
\balignt
\|A_\beta\tilde{g}_{h,\rho,\eps}\|_{\infty}
    \leq \|\Xi g_{h,\rho,\eps}\|_{\infty} \sqrt{2}\E[A_\beta(Y)]
    \leq \widetilde{\MP} \defeq \MP' \sqrt{2}\E[A_\beta(Y)],
\ealignt
where $\widetilde{\MP}$ is independent of $h,\eps,$ and $\rho$.
Thus for any $\sset$, Cauchy-Schwarz, our bound \cref{eq:tilted-approx-stein-bound},
the triangle inequality, and the fact that $\|\grad\log
p_{\sset}/A_\beta\|_{\infty} \le 1$ and $\|\grad\log p/A_\beta\|_{\infty} \le 1$ imply
\balignt
\|\textfrac{L}{m}\opsub{\sset}\tilde{g}_{h,\rho,\eps} &- \langevin{\tilde{g}_{h,\rho,\eps}}\|_{\infty}
  = \|\inner{\textfrac{L}{m}\grad\log p_{\sset} - \grad\log p}{\tilde{g}_{h,\rho,\eps}}\|_{\infty} \\
  &\le \|(\textfrac{L}{m}\grad\log p_{\sset} - \grad\log p)/A_\beta\|_{\infty} \|A_\beta\tilde{g}_{h,\rho,\eps}\|_{\infty} \\
  &\le \widetilde{\MP} (\textfrac{L}{m} \|\grad\log p_{\sset}/A_\beta\|_{\infty} +
  \|\grad\log p/A_\beta\|_{\infty})
  \le (\textfrac{L}{m} + 1)\widetilde{\MP}.
\ealignt
Thus, the triangle inequality and our error bounds \cref{eq:rho-approx-error,eq:g0-approx-error,eq:geps-approx-error} yield
\begin{align}
\|\langevin{\tilde{g}_{h,\rho,\eps}}\|_{\infty} 
    &\le \supnorm{\langevin{g_{h}} - \langevin{g_{h,\rho}}}
    + \supnorm{\langevin{g_{h,\rho}} - \langevin{g_{h,\rho,\eps}}}
    + \supnorm{\langevin{g_{h,\rho,\eps}}-\langevin{\tilde{g}_{h,\rho,\eps}}} 
    + \supnorm{\langevin{g_{h}}} \\
    &\leq C_1\rho + C_0 + C_\rho \eps + 2 \qtext{and}\\
\|\opsub{\sset}\tilde{g}_{h,\rho,\eps}\|_{\infty} 
    &\le \|\opsub{\sset}\tilde{g}_{h,\rho,\eps}
- \textfrac{m}{L}\langevin{\tilde{g}_{h,\rho,\eps}}\|_{\infty} + \textfrac{m}{L}\|\langevin{\tilde{g}_{h,\rho,\eps}}\|_{\infty} \\
    &\le 
    \tilde{C}_{\eps,\rho} \defeq 
(1+\textfrac{m}{L})\widetilde{\MP} +
\textfrac{m}{L} (C_1\rho + C_0 + C_\rho \eps + 2)
\end{align}
for each $\sset$.
Therefore, $\tilde{g}_{h,\rho,\eps} \in \max(C_{\eps,\rho}, \tilde{C}_{\eps,\rho}) \gset_{b,n}$.

\paragraph{Upper bounding the IPM $d_{\hset}$} 
Finally, we combine the triangle inequality and our approximation bounds \cref{eq:rho-approx-error,eq:g0-approx-error,eq:geps-approx-error} once more to conclude
\begin{align}
&d_{\hset}(Q_n, P) 
\defeq
\sup_{h\in\hset}|Q_n(h) - P(h)|
= \sup_{h\in\hset}|Q_n(\langevin{g_h})| 
\\
&\leq 
\sup_{h\in\hset}|Q_n(\langevin{\tilde{g}_{h,\rho,\eps}})| 
+ |Q_n(\langevin{\tilde{g}_{h,\rho,\eps}} - \langevin{g_{h,\rho,\eps}})| 
+ |Q_n(\langevin{g_{h,\rho}} - \langevin{g_{h,\rho,\eps}})|
+ |Q_n(\langevin{g_h} - \langevin{g_{h,\rho}})|\\
&\leq
\sup_{h\in\hset}|Q_n(\langevin{\tilde{g}_{h,\rho,\eps}})| + C_\rho\eps + C_0 Q_n(K_\eps^c) + C_1 \rho\\
&\leq 
\max(C_{\eps,\rho}, \tilde{C}_{\eps,\rho})
\langstein{Q_n}{\gset_{b,n}} + (C_0 + C_\rho)\eps + C_1\rho.
\end{align}
Since $\eps$ and $\rho$ were arbitrary, 
whenever $\langstein{Q_n}{\gset_{b,n}} \to 0$, we have 
$d_{\hset}(Q_n, P) \to 0$ and hence $Q_n\Rightarrow P$.

\subsection{Classical Stein set}
Suppose $\gset_n$ satisfies \cref{classical-set}, and consider
$\ksteinsetnorm{k}{\twonorm{\cdot}}$ for $k(x,y) = \Phi(x-y) \defeq (1 +
\twonorm{\Gamma(x-y)}^2)^{\beta}$ with $\beta < 0$ and $\Gamma \succ 0$.  Since $\grad^s
\Phi(0)$ is bounded for $s\in\{0,2,4\}$, \citep[Cor. 4.36]{Christmann2008}
implies that $\ksteinsetnorm{k}{\twonorm{\cdot}} \subseteq c_0 \gset_n$ for
some $c_0$.  The result now follows since
$\ksteinsetnorm{k}{\twonorm{\cdot}}$ also satisfies \cref{kernel-set}.

\subsection{Graph Stein set}
If $\gset_n$ satisfies \cref{graph-set}, the result follows as $\gset_n$ contains the classical Stein set $\steinset$.

\section{Proof of \cref{ssds-detect-bounded-sd-non-convergence}: SSDs detect bounded SD non-convergence}
\label{sec:proof-ssds-detect-bounded-sd-non-convergence}
Since $\opstein{Q_n}{\gset_{b,n}} \not\to 0$, there exists $\eps > 0$ such that $\opstein{Q_n}{\gset_{b,n}} > \eps$ infinitely often (i.o.).
Fix any such $\eps$.
For each $n$, choose $h_n=\langevin{g_n}$ for $g_n \in \gset_{b,n}$ satisfying $Q_n(h_n) \geq \opstein{Q_n}{\gset_{b,n}} - \eps/2$.
Then since $\operator = {L\choose m}^{-1}\frac{L}{m}\sum_{\sset\in\equisubs{L}{m}}\opsub{\sset}$,
\balignt
\opstein{Q_n}{\gset_{b,n}} - \eps/2
    &\leq Q_n(h_n) - {L\choose m}^{-1}\sum_{\sset\in\equisubs{L}{m}}\mu_{n\sset}(\opsub{\sset}g_n) + {L\choose m}^{-1}\sum_{\sset\in\equisubs{L}{m}}\mu_{n\sset}(\opsub{\sset}g_n) \\
    &\leq {L\choose m}^{-1}\sum_{\sset\in\equisubs{L}{m}}(\frac{L}{m}Q_n(\opsub{\sset}g_n) - \mu_{n\sset}(\opsub{\sset}g_n)) + \ssd{Q_n}.
\ealignt
Moreover, since $\supnorm{\opsub{\sset}g_n}\leq 1$ for all
$\sset\in\equisubs{L}{m}$ and $n$, \cref{bounded_func_convergence}, proved
in \cref{sec:proof-bounded_func_convergence}, implies that
$\frac{L}{m}Q_n(\opsub{\sset}g_n)- \mu_{n\sset}(\opsub{\sset}g_n) \toas 0$
for each $\sset$.
\begin{lemma}[Bounded function convergence] \label{bounded_func_convergence}
Fix any triangular array of points $(x_i^n)_{i\in[n], n \geq 1}$ in $\reals^d$,
and, for each $n\geq 1$, define the measures 
\balignt
\nu_n = \frac{1}{n} \sum_{i=1}^n \delta_{x_i^n}
\qtext{and} 
\tilde{\nu}_{n} = \frac{1}{n} \sum_{i=1}^n \frac{B_{i}}{\tau}\delta_{x_i^n}
\ealignt
where $B_i \distiid \Ber(\tau)$ are independent Bernoulli random variables with
$\P(B_i = 1) = \tau$.
If $\supnorm{h_n}\leq 1$ for each $n$, then, with probability $1$,
\balignt
|\tilde{\nu}_{n}(h_n) - \nu_n(h_n)| \leq \tau^{-1}\sqrt{\frac{\log(n)+2\log(\log(n)))}{2n}}
\ealignt
for all $n$ sufficiently large.
Hence, $\tilde{\nu}_{n}(h_n) - \nu_n(h_n) \toas 0$.
\end{lemma}
Hence
\balignt
\P(\ssdn{Q_n} \not\to 0)
&\geq \P(\ssdn{Q_n} > \eps/2 \io) \\
&\geq \P(Q_n(\opsub{\sset}g_n) - \mu_{n\sset}(\opsub{\sset}g_n) <
\frac{\eps}{2} \text{ eventually}, \forall \sset) = 1
\ealignt
as advertised.

\subsection{Proof of \cref{bounded_func_convergence}: Bounded function convergence}
\label{sec:proof-bounded_func_convergence}
The result will follow from the following lemma which establishes rates of convergence for subsampled measure expectations to their non-subsampled counterparts.
\begin{lemma}\label{mu_concentration}
Under the notation of \cref{bounded_func_convergence}, for any $a \in [1,2]$, $\delta\in(0,1)$, and $h:\reals^d\to \reals$,
\balignt
\tilde{\nu}_{n}(h) - \nu_n(h) &\leq \frac{\tau^{-1}\sqrt{\half\log(1/\delta)}}{n^{1-1/a}}(\nu_n(|h|^a))^{1/a}
\qtext{with probability at least}
1-\delta 
\qtext{and} \\
\nu_n(h) - \tilde{\nu}_{n}(h) &\leq \frac{\tau^{-1}\sqrt{\half\log(1/\delta)}}{n^{1-1/a}}(\nu_n(|h|^a))^{1/a}
\qtext{with probability at least}
1-\delta.
\ealignt
\end{lemma}
\begin{proof}
Fix any $a \in [1,2]$, $\delta\in(0,1)$, and $h:\reals^d\to \reals$.
Since 
\balignt
\tilde{\nu}_{n}(h) = \frac{1}{n}\sum_{i=1}^n \frac{B_{i}}{\tau} h(x_i^n)
\ealignt
is an average of independent variables $\tau^{-1}B_{i}
h(x_i^n) \in \{0, \tau^{-1}h(x_i^n)\}$
with $\E[\tilde{\nu}_{n}(h)] = \nu_n(h)$, Hoeffding's inequality~\cite[Thm.~2]{Hoeffding} implies
\balignt
\tilde{\nu}_{n}(h) - \nu_n(h) 
    &\leq 
    \tau^{-1}\sqrt{\log(1/\delta)\frac{1}{2n^2}\sum_{i=1}^n h(x_i^n)^2} \qtext{with probability at least} 1-\delta \qtext{and}\\
\nu_n(h) - \tilde{\nu}_{n}(h)
    &\leq \tau^{-1}\sqrt{\log(1/\delta)\frac{1}{2n^2}\sum_{i=1}^n h(x_i^n)^2} \qtext{with probability at least} 1-\delta.
\ealignt
Moreover, since $\twonorm{\cdot} \leq \norm{\cdot}_a$, 
we have
$\sqrt{\sum_{i=1}^n h(x_i^n)^2/n^2} 
\leq (\sum_{i=1}^n |h(x_i^n)|^a/n^a)^{1/a}$, and
the advertised result follows.
\end{proof}
\\
By \cref{mu_concentration} with $a=2$,
\balignt
\sum_{n=1}^\infty \P(|\nu_{n}(h_n) - \tilde{\nu}_n(h_n)| \geq \tau^{-1}\sqrt{\frac{\log(1/\delta_n)}{2n}})
\leq \sum_{n=1}^\infty \delta_n < \infty
\ealignt
for $\delta_n = 1/(n\log^2(n))$. 
The result now follows from the Borel-Cantelli lemma.

\section{Proof of \cref{ssd-tightness}: Coercive SSDs enforce tightness}
\label{sec:proof-ssd-tightness}
Let $f(x) = \min_{\sset\in\equisubs{L}{m}} \frac{L}{m}\opsubarg{\sset}{g}{x}$.
Since $f$ is bounded below, $C = \inf_{x\in\reals^d} f(x)$ is finite.
Define
\[
\gamma(r) \defeq \inf\{ f(x) - C : \twonorm{x} \geq r\},
\] 
so that $\gamma$ is nonnegative, coercive, and non-decreasing, as $f$ is coercive.
Since $(Q_n)_{n=1}^\infty$ is not tight, there exist $\epsilon > 0$ and $R>0$ such that
$\limsup_n Q_n(\twonorm{X} > R) \geq \eps$ and $\gamma(R)\eps + C > 0$.
Moreover, since $\gamma$ is non-decreasing and nonnegative, Markov's inequality gives
\balignt
Q_n(\twonorm{X} > R)
\leq Q_n(\gamma(\twonorm{X}) > \gamma(R)) 
\leq \Esubarg{Q_n}{\gamma(\twonorm{X})}/\gamma(R)
\leq (Q_n(f)-C)/\gamma(R).
\ealignt
Meanwhile, our assumption on $g$ and the SSD subset representation \cref{ssd_subset} imply that, surely,
\balignt
Q_n(f) = \frac{1}{n}\sum_{i=1}^n f(x_i)
    \leq \frac{1}{n}\sum_{i=1}^n \frac{L}{m}\opsubarg{\sset_i}{g}{x_i}
    \leq \ssdn{Q_n}.
\ealignt
Hence, $\ssdn{Q_n}$ surely does not converge to zero, as
\balignt
\limsup_n \ssdn{Q_n}
    \geq \gamma(R)\limsup_n Q_n(\twonorm{X} > R) + C
    \geq \gamma(R)\eps + C > 0.
\ealignt

\section{Proof of \cref{coercive-ssds-detect-non-convergence}: Coercive SSDs detect non-convergence}
\label{sec:proof-coercive-ssds-detect-non-convergence}

We consider each Stein set candidate in turn.

\paragraph{Kernel Stein set}
Suppose $\gset_n$ satisfies \cref{kernel-set} for one of the specified
kernels, $k_1(x,y) = \Phi_1(x-y)$ or $k_2(x,y)=\Phi_2(x-y)$, with $\Gamma =
I_d$.

We have $\hat{\Phi}_1$ and $\hat{\Phi}_2$ are non-vanishing by
\cite[Thm. 8.15]{Wendland2004} and \cite[Lem. 7]{ChenMaGoFXOa18},
respectively. Moreover, we have for all $x,y\in\reals^d$
\balign
\inner{\grad \log p(x)-\grad \log p(y)}{x-y}
&= \textfrac{L}{m}\textstyle{L \choose m}^{-1}\sum_{\sset}
  \inner{\grad \log p_\sset(x)-\grad \log p_\sset(y)}{x-y} \\ \notag
&\le -\kappa\twonorm{x-y}^2 + r. \notag
\ealign
Hence if $Q_n \not\Rightarrow P$, then, by
\cref{bounded-sds-detect-tight-non-convergence}, either
$\langstein{Q_n}{\gset_{b,n}} \not\to 0$ or
$(Q_n)_{n=1}^\infty$ is not tight.

If $\langstein{Q_n}{\gset_{b,n}} \not\to 0$, then, with probability $1$,
$\langssdn{Q_n} \not\to 0$ by \cref{ssds-detect-bounded-sd-non-convergence}.

Now suppose $(Q_n)_{n=1}^\infty$ is not tight, and fix any $\sset\in\equisubs{L}{m}$.
Consider first the kernel $k_1$.
Since $\frac{L}{m}\grad \log p_\sigma$ has at most linear growth and
satisfies distant dissipativity, the proof of \citep[Lem. 16]{GorhamMa17} constructs a
function $g\in \gset_n$ that is independent of the choice of $\sigma$ and
satisfies $\frac{L}{m}\opsub{\sigma} g \geq f_\sigma$ for some coercive bounded-below
$f_\sigma$. Similarly, the same conclusion holds for the kernel $k_2$ by the proof of
\cite[Thm. 3]{ChenMaGoFXOa18}.
Since $\equisubs{L}{m}$ has finite cardinality,
we have $\frac{L}{m}\opsub{\sigma}g \geq f$ for a common coercive bounded-below function $f(x) \defeq
\min_{\sset} f_\sigma(x)$.
Therefore, surely, $\langssdn{Q_n} \not\to 0$ by \cref{ssd-tightness}.

To extend this result to any $\Gamma \succ 0$, fix some $\Gamma \succ 0$.
For any distribution $P$ on $\reals^d$, let us write $\Gamma^{-1}P$ to
represent the distribution of $\Gamma^{-1} Z$ when $Z\sim P$.
Let $p_{\Gamma}$ be the density $\Gamma^{-1}P$.
Then $p_{\Gamma}(x) = \text{det}(\Gamma)\grad\log p(\Gamma x)$ and
$\grad\log p_{\Gamma}(x) = \Gamma\grad\log p(\Gamma x)$, and for any
$\sset\in\equisubs{L}{m}$, the analog $p_{\Gamma,\sset}$ of $p_{\Gamma}$ satisfies
$p_{\Gamma,\sset}(x) = \text{det}(\Gamma)\grad\log p_\sset(\Gamma x)$ and
$\grad\log p_{\Gamma,\sset}(x) = \Gamma\grad\log p_\sset(\Gamma x)$.
By the same argument made in \cite[Lem. 4]{ChenBaFXGoGiMaOa19}, we have that
$\grad\log p_{\Gamma}$ is Lipschitz and $\grad\log p_{\Gamma,\sset}$
satisfies distant dissipativity. And since
\balign
\frac{\twonorm{\grad\log p_{\Gamma,\sset}(x)}}{1 +
    \twonorm{x}}
  = \frac{\twonorm{\Gamma \grad\log p_\sset(\Gamma x)}}{1 + \twonorm{\Gamma x}} \frac{1 +
    \twonorm{\Gamma x}}{1 + \twonorm{x}}
  \le \opnorm{\Gamma}(1 + \opnorm{\Gamma}) \frac{\twonorm{\grad\log p_\sset(\Gamma x)}}{1 + \twonorm{\Gamma
      x}} \notag
\ealign
is uniformly bounded, we can apply the same argument discussed in \cite[Lem. 4]{ChenBaFXGoGiMaOa19},
i.e., make a global change of coordinates $x \mapsto \Gamma^{-1} x$ and
then invoke \cref{coercive-ssds-detect-non-convergence} for $\Gamma^{-1}
P$ and $\Gamma^{-1}Q_n$ with a non-preconditioned kernel, thereby concluding
the proof.

\paragraph{Classical Stein set}
Suppose $\gset_n = \steinset$ satisfies \cref{classical-set}. By the proof
of \cref{bounded-sds-detect-tight-non-convergence}, for $\Gamma = I$ and any
$\beta\in(-1, 0)$, there is a constant $c_0 > 0$ such that the kernel Stein
set $\ksteinsetnorm{k}{\twonorm{\cdot}} \subseteq c_0 \gset_n$. Hence
$\mathcal{SS}(Q_n,\langevin{},{\ksteinsetnorm{k}{\twonorm{\cdot}}}) \le c_0
\langssdn{Q_n}$ for all $n$ implying the result.

\paragraph{Graph Stein set}
Suppose $\gset_n$ satisfies \cref{graph-set}. Then the result follows as
$\gset_n$ contains the classical Stein set $\steinset$.

\section{Proof of \cref{ssvgd-converges}: Wasserstein convergence of SVGD and SSVGD}
\label{sec:proof-ssvgd-converges}

\subsection{Additional notation} 
For each $\eps > 0$ and collection of $n$ points $(x_i^n)_{i=1}^n$ with associated discrete measure $\nu_n = \frac{1}{n}\sum_{i=1}^n \delta_{x_i^n}$, we define the random one-step SSVGD mapping
\balignt
T_{\nu_n, \eps, n}^m(x) 
    &= x + \eps \frac{1}{n}\sum_{j=1}^n\frac{L}{m}\grad \log p_{\sig_{j}}(x^n_j)k(x^n_j,x) + \grad_{x^n_j} k(x^n_j, x)
\ealignt
for $(\sig_{j})_{j=1}^n$ independent uniformly random size-$m$ subsets of $[L]$. 
We also let $\Phi_{\eps,n}^m(\mu)$ denote the random distribution of $T_{\nu_n,\eps,n}^m(X)$ when $X\sim \mu$.

\subsection{Proof of \cref{ssvgd-converges}}
We will prove each convergence claim by induction on $r \geq 0$.
\paragraph{Inductive proof of $\lswass{1}(Q_{n,r}, Q_{\infty,r}) \to 0$}
For our base case we have $\lswass{1}(Q_{n,0}, Q_{\infty,0}) \to 0$ by assumption.

Now, fix any $r \geq 0$ and assume $\lswass{1}(Q_{n,r}, Q_{\infty,r}) \to 0$, so that $c_0(1+\twonorm{\cdot})$ is uniformly $Q_{n,r}$-integrable and $Q_{n,\infty}$-integrable by \citep[Proof~of~Cor.~1]{ekisheva2006transportation}.
Therefore, there exists a constant $C' > 0$ such that
\balignt\label{eq:ui-mean-bound}
 \sup_{n\geq 1} 1+\eps_r c_1(1+ Q_{n,r}(\twonorm{\cdot})) + \eps_r c_2 (1+ Q_{\infty,r}(\twonorm{\cdot})) \leq C'.
\ealignt
Now, note that 
\balignt
\lswass{1}(Q_{n,r+1}, Q_{\infty,r+1})
    &= \lswass{1}(\Phi_{\eps_r}(Q_{n,r}), \Phi_{\eps_r}(Q_{\infty,r})).
\ealignt
To control this expression, we provide a lemma, proved in \cref{sec:proof-svgd-pseudolip}, which establishes the pseudo-Lipschitzness of the one-step SVGD mapping $\Phi_\eps$.
\begin{lemma}[Wasserstein pseudo-Lipschitzness of SVGD]
\label{svgd-pseudolip}
Suppose that, for some $c_1,c_2 > 0$,
\balignt
&\sup_{z\in\reals^d} \opnorm{\grad_z (\grad \log p(x)k(x,z) + \grad_x k(x,z))} \leq c_1(1+\twonorm{x}) \qtext{and}\\
&\sup_{x\in\reals^d} \opnorm{\grad_x (\grad \log p(x)k(x,z) + \grad_x k(x,z))} \leq c_2(1+\twonorm{z}).
\ealignt
Then, for any $\eps > 0$ and probability measures $\mu, \nu$,
\balignt
\lswass{1}(\Phi_\eps(\mu),\Phi_\eps(\nu))
    \leq \lswass{1}(\mu,\nu) 
    (1+\eps c_1(1+ \mu(\twonorm{\cdot})) + \eps c_2 (1+ \nu(\twonorm{\cdot}))).
\ealignt
\end{lemma}
Our pseudo-Lipschitz assumptions \cref{eq:M1-lin-growth} and \cref{svgd-pseudolip} imply
\begin{talign}
\lswass{1}(\Phi_{\eps_r}(Q_{n,r}), \Phi_{\eps_r}(Q_{\infty,r}))
    &\leq 
    \lswass{1}(Q_{n,r}, Q_{\infty,r})
    (1+\eps_r c_1(1+ Q_{n,r}(\twonorm{\cdot})) + \eps_r c_2 (1+ Q_{\infty,r}(\twonorm{\cdot}))) \\
    &\leq 
    C'\lswass{1}(Q_{n,r}, Q_{\infty,r}) \to 0,
\end{talign}
proving our first claim.

\paragraph{Inductive proof of $\lswass{1}(Q_{n,r}^m, Q_{n,r}) \to 0$}
For our base case we have,  $\lswass{1}(Q_{n,0}^m, Q_{n,0}) = 0$.

Now fix any $r\geq0$, let $\event$ be the event on which  $\lswass{1}(Q_{n,r}^m, Q_{n,r}) \to 0$
 as $n\to \infty$, and assume $\P(\event) = 1$.
 Since $\lswass{1}(Q_{n,r}, Q_{\infty,r}) \to 0$, on $\event$ we find that $\lswass{1}(Q_{n,r}^m, Q_{\infty,r}) \to 0$ and hence $c_0(1+\twonorm{\cdot})$ is uniformly $Q_{n,r}^m$-integrable and uniformly $Q_{n,r}$-integrable by \citep[Proof~of~Cor.~1]{ekisheva2006transportation}.
 Therefore, on $\event$, there exists a constant $C$ such that
 \balignt\label{eq:ui-mean-bound}
 \sup_{n\geq 1} 1+\eps_r c_1(1+ Q_{n,r}^m(\twonorm{\cdot})) + \eps_r c_2 (1+ Q_{n,r}(\twonorm{\cdot})) \leq C.
 \ealignt

By the triangle inequality,
\balignt
\lswass{1}(Q_{n,r+1}^m, Q_{n,r+1})
    &= \lswass{1}(\Phi_{\eps_r,n}^m(Q_{n,r}^m), \Phi_{\eps_r}(Q_{n,r})) \\
    &\leq 
    \lswass{1}(\Phi_{\eps_r,n}^m(Q_{n,r}^m), \Phi_{\eps_r}(Q_{n,r}^m))
    + 
    \lswass{1}(\Phi_{\eps_r}(Q_{n,r}^m), \Phi_{\eps_r}(Q_{n,r})). \label{eq:wass-ssvgd-svgd-triangle}
\ealignt
On $\event$, our growth assumptions \cref{eq:M0-lin-growth}, the uniformly $Q_{n,r}^m$-integrability of $c_0(1+\twonorm{\cdot})$, and the following lemma, proved in \cref{sec:proof-one-step-ssvgd-convergence}, establish that the Wasserstein distance $\lswass{1}(\Phi_{\eps_r,n}^m(Q_{n,r}^m), \Phi_{\eps_r}(Q_{n,r}^m))$ between one step of SSVGD  and one step of SVGD from a common starting point converges to $0$ almost surely as $n$ grows.
\begin{lemma}[One-step convergence of SSVGD to SVGD]
\label{one-step-ssvgd-convergence}
Fix any triangular array of points $(x_i^n)_{i\in[n], n \geq 1}$ in $\reals^d$,
and define the discrete probability measures 
$\nu_n = \frac{1}{n} \sum_{i=1}^n \delta_{x_i^n}.$
Suppose $\grad \log p_{\sset}(\cdot) k(\cdot,z)$ is continuous for each $z\in\reals^d$ and $\sset \in \equisubs{L}{m}$ and let
\balignt
f_0(x) &\defeq \sup_{z\in\reals^d,\sset \in \equisubs{L}{m}} \infnorm{\grad
  \log p_{\sset}(x)} |k(x,z)|, \\
f_1(x) &\defeq \sup_{z\in\reals^d,\sset
  \in \equisubs{L}{m}} \opnorm{\grad_x (\grad \log p_{\sset}(x)k(x,z))}.
\ealignt 
If $f_0$ is $\nu_n$-uniformly integrable and $f_0,f_1$ are bounded on each compact
set, then, for any $\eps > 0$,
$\lswass{1}(\Phi_{\eps,n}^m(\nu_n), \Phi_\eps(\nu_n)) \toas 0$ as $n\to\infty$.
\end{lemma}
In addition, on $\event$, our pseudo-Lipschitz assumptions \cref{eq:M1-lin-growth} and \cref{svgd-pseudolip} imply
\begin{talign}
\lswass{1}(\Phi_{\eps_r}(Q_{n,r}^m), \Phi_{\eps_r}(Q_{n,r}))
    &\leq 
    \lswass{1}(Q_{n,r}^m, Q_{n,r})
    (1+\eps c_1(1+ Q_{n,r}^m(\twonorm{\cdot})) + \eps c_2 (1+ Q_{n,r}(\twonorm{\cdot}))) \\
    &\leq 
    C\lswass{1}(Q_{n,r}^m, Q_{n,r}) \to 0.
\end{talign}
Hence, on $\event$, $\lswass{1}(Q_{n,r+1}^m, Q_{n,r+1}) \toas 0$, proving our second claim.

\subsection{Proof of \cref{svgd-pseudolip}: Wasserstein pseudo-Lipschitzness of SVGD}
\label{sec:proof-svgd-pseudolip}
Assume that $\mu$ and $\nu$ have integrable means (or else the advertised claim is vacuous), and select $(X',Z')$ to be an optimal $1$-Wasserstein coupling of $(\mu, \nu)$.
The triangle inequality, Jensen's inequality, and our pseudo-Lipschitzness assumptions imply that
\balignt
&\twonorm{T_{\mu,\eps}(x) - T_{\nu,\eps}(z)} \\
    &\leq \twonorm{x - z} \\
    &+ \eps \twonorm{\Earg{\grad\log p(X')k(X',x) + \grad_{x'} k(X', x) - (\grad\log p(X')k(X',z) + \grad k(X', z))}} \\
    &+ \eps \twonorm{\Earg{\grad\log p(X')k(X',z) + \grad_{x'} k(X', z) - (\grad\log p(Z')k(Z',z) + \grad_{z'} k(Z', z))}} \\
   &\leq \twonorm{x - z} (1+\eps c_1(1+ \E[\twonorm{X'}])) + \eps c_2 \Earg{\twonorm{X' - Z'}}(1+ \twonorm{z}) \\
   &= \twonorm{x - z} (1+\eps c_1(1+ \mu(\twonorm{\cdot})) + \eps c_2 \lswass{1}(\mu,\nu)(1+ \twonorm{z}).
\ealignt
Since $T_{\mu,\eps}(X') \sim \Phi_\eps(\mu)$ and $T_{\nu,\eps}(Z') \sim \Phi_\eps(\nu)$, we conclude that
\balignt
\lswass{1}(\Phi_\eps(\mu),\Phi_\eps(\nu))
    &\leq
    \E[\twonorm{T_{\mu,\eps}(X') - T_{\nu,\eps}(Z')}] \\
    &\leq \E[\twonorm{X' - Z'}] (1+\eps c_1(1+ \mu(\twonorm{\cdot})) + \eps c_2 \lswass{1}(\mu,\nu)(1+ \E[\twonorm{Z'}])\\
    &= \lswass{1}(\mu,\nu) 
    (1+\eps c_1(1+ \mu(\twonorm{\cdot})) + \eps c_2 (1+ \nu(\twonorm{\cdot}))).
\ealignt

\subsection{Proof of \cref{one-step-ssvgd-convergence}: One-step convergence of SSVGD to SVGD}
\label{sec:proof-one-step-ssvgd-convergence}
Note that the random one-step SSVGD mapping takes the form
\balignt
T_{\nu_n, \eps, n}^m(x) 
    &= x + \eps \nu_n(\grad_{x^n_j} k(\cdot, x)) + \eps {L\choose m}^{-1}\sum_{\sset\in\equisubs{L}{m}}\nu_{n\sset}(\grad \log p_{\sset}(\cdot) k(\cdot,x))
\ealignt
for $\nu_{n\sset} = {L\choose m}\textfrac{L}{m}\frac{1}{n}\sum_{j=1}^nB_{j\sset}\delta_{x^n_j}$
and $B_{j\sset} = \indic{\sset = \sset_{j}}$.
Moreover, by Kantorovich-Rubinstein duality, we may write the $1$-Wasserstein distance as
\balignt
\lswass{1}(&\Phi_{\eps,n}^m(\nu_n), \Phi_\eps(\nu_n)) \\
    &= \sup_{f : M_1(f) \leq 1} \Phi_{\eps,n}^m(\nu_n)(f) - \Phi_\eps(\nu_n)(f) \\
    &= \sup_{f : M_1(f) \leq 1} \frac{1}{n} \sum_{i=1}^n f(T_{\nu_n, \eps, n}^m(x^n_i)) - f(T_{\nu_n, \eps}(x^n_i)) \\
    &\leq \frac{1}{n} \sum_{i=1}^n \twonorm{T_{\nu_n, \eps, n}^m(x^n_i) - T_{\nu_n, \eps}(x^n_i)}\\
    &= {L\choose m}^{-1}\frac{\eps }{n} \sum_{i=1}^n 
    \twonorm{\sum_{\sset} \textfrac{L}{m}\nu_n(\grad \log p_\sset(\cdot)
      k(\cdot,x^n_i)) - \nu_{n\sset}(\grad \log p_\sset(\cdot) k(\cdot,x^n_i))}\\
    & \leq {L\choose m}^{-1}\sum_{\sset} \frac{\eps \sqrt{d}}{n} \sum_{i=1}^n 
    \infnorm{\textfrac{L}{m}\nu_n(\grad \log p_\sset(\cdot) k(\cdot,x^n_i)) - \nu_{n\sset}(\grad \log p_\sset(\cdot) k(\cdot,x^n_i))} \\
    &\leq \eps\sqrt{d}{L\choose m}^{-1}\sum_{\sset}
\sup_{h \in \hset_f} |\nu_{n\sset}(h) - \textfrac{L}{m}\nu_n(h)|. \label{eq:one-step-ssvgd-convergence-bound}
\ealignt
where we have used the triangle inequality and norm relation
$\twonorm{\cdot}\leq \sqrt{d}\infnorm{\cdot}$ in the penultimate display and
$\hset_f$ is defined in the statement of \cref{random_measure_convergence}.

For each $\sset\in\equisubs{L}{m}$, since $|f_0|$ is uniformly
$\nu_n$-integrable, and $\nu_{n\sset}(|f_0| I_K) \leq {L\choose m}\frac{L}{m} \nu_n(|f_0| I_K)$ for every compact set $K$, we find that $|f_0|$ is uniformly $\nu_{n\sset}$-integrable for each $\sset$.
Letting $I_{B_R}(x) = \indic{\twonorm{x}\leq R}$, for each $\sset$, since $\nu_{n\sset}(hI_{B_R}) - \textfrac{L}{m}\nu_n(hI_{B_R}) \toas 0$ for 
 any $R > 0$ and any bounded $h$ by \cref{bounded_func_convergence}, we have $\sup_{h \in \hset_f}
|\nu_{n\sset}(h) - \textfrac{L}{m}\nu_n(h)| \toas 0$ by \cref{random_measure_convergence}.
The result now follows from the bound \cref{eq:one-step-ssvgd-convergence-bound}.

\end{document}